\documentclass{article}

\PassOptionsToPackage{numbers,compress,sort}{natbib}

\usepackage[preprint]{neurips_2020}
\newcommand{\Code}{\url{https://github.com/yaohungt/Pointwise_Dependency_Neural_Estimation
}}
\usepackage{graphicx}
\usepackage{xcolor}
\usepackage[utf8]{inputenc} 
\usepackage[T1]{fontenc}    
\usepackage[colorlinks]{hyperref}       

\usepackage{url}
\usepackage{multirow}
\usepackage{booktabs}       
\usepackage{amsfonts}       
\usepackage{nicefrac}       
\usepackage{microtype}      
\usepackage{mathtools}
\usepackage{amsmath,amsthm}
\usepackage{amssymb}
\usepackage{thmtools,thm-restate}

\theoremstyle{definition} 
\newtheorem{theorem}{Theorem}
\newtheorem{defn}{Definition}
\newtheorem{prop}{Proposition}
\newtheorem{lemma}{Lemma}
\newtheorem{assumption}{Assumption}

\DeclarePairedDelimiterX{\infdivx}[2]{(}{)}{%
  #1\;\delimsize\|\;#2%
}
\newcommand{\xxspace}{\mathcal{X}}
\newcommand{\yyspace}{\mathcal{Y}}
\newcommand{\defeq}{\vcentcolon=}

\newcommand{\eps}{\varepsilon}
\newcommand{\Exp}{\mathbb{E}}
\newcommand{\kldiv}{D_{\text{KL}}\infdivx}

\newcommand{\RR}{\mathbb{R}}
\newcommand{\FF}{\mathcal{F}}
\newcommand{\covering}{\mathcal{N}}
\newcommand{\Nat}{\mathbb{N}}

\title{\Large Neural Methods for Point-wise Dependency Estimation}

\author{%
  Yao-Hung Hubert Tsai$^1$, Han Zhao$^2$\thanks{Work done at Carnegie Mellon University.}, \\ {\bf Makoto Yamada}$^{34}${\bf,  Louis-Philippe Morency}$^1${\bf , Ruslan Salakhutdinov}$^1$ \\
 $^1$Carnegie Mellon University, $^2$D.E.\ Shaw \& Co., $^3$ Kyoto University, $^4$RIKEN AIP 
}

\begin{document}

\maketitle

\begin{abstract}
Since its inception, the neural estimation of mutual information (MI) has demonstrated the empirical success of modeling expected dependency between high-dimensional random variables. However, MI is an aggregate statistic and cannot be used to measure point-wise dependency between different events. In this work, instead of estimating the expected dependency, we focus on estimating point-wise dependency (PD), which quantitatively measures how likely two outcomes co-occur. 
We show that we can naturally obtain PD when we are optimizing MI neural variational bounds. However, optimizing these bounds is challenging due to its large variance in practice. To address this issue, we develop two methods (free of optimizing MI variational bounds): Probabilistic Classifier and Density-Ratio Fitting.
We demonstrate the effectiveness of our approaches in 1) MI estimation, 2) self-supervised representation learning, and 3)
cross-modal retrieval task. 
\end{abstract}

\section{Introduction}

Mutual Information (MI) measures the average statistical dependency between two random variables, and it has found abundant applications in practice, such as feature selection~\cite{peng2005feature,chen2018learning}, interpretable factor discovery~\cite{chen2016infogan,tsai2018learning}, genetic association studies~\cite{zhang2012inferring}, to name a few. Recent work~\cite{belghazi2018mine,poole2019variational} proposed to use neural networks with gradient descent to estimate MI, which empirically scales better in high-dimension settings as compared to classic approaches (e.g., Kraskov (KSG)~\cite{kraskov2004estimating} estimator), which are known to suffer from the curse of dimensionality. Inspired by this line of work, we take a step further to present neural methods for point-wise dependency (PD) estimation.
At a colloquial level, PD serves to understand the instance-level dependency between a pair of events taken by two random variables, which gives us a fine-grained understanding of the outcome. Formally, it can be realized as the ratio between likelihood of their co-occurrence to the likelihood of the product events: $p(x,y)/p(x)p(y)$ with $x$ and $y$ being the corresponding outcomes.

At first glance, it may seem straightforward to estimate PD by adopting prior density ratio estimation approaches~\cite{sugiyama2012density3,sugiyama2012density} to directly calculate the ratio between $p(x,y)$ and $p(x)p(y)$. Nonetheless, for the sake of tractability, previous methods are mainly kernel-based approaches that might be inadequate to scale to high-dimensional and complex-structured data. In this work, we introduce approaches for PD estimation that leverage the recent advances in rich and flexible neural networks. We show that we can naturally obtain PD when we are optimizing MI neural variational bounds~\cite{belghazi2018mine,poole2019variational}. 
However, estimating these MI bounds often results in inevitably large variance~\cite{song2019understanding}. To address this concern, we develop two data-driven approaches: {\em Probabilistic Classifier} and {\em Density-Ratio Fitting}. {\em Probabilistic Classifier} turns PD estimation into a supervised binary classification task, where we train a classifier to distinguish the observed joint distribution from the product of marginal distribution. This approach adopts cross-entropy loss using neural networks, which is favorable for optimization and exhibits a stable training trajectory with less variance. {\em Density-Ratio Fitting} seeks to minimize the least-square difference between the true and the estimated PD. Its objective involves no logarithm and exponentiation; hence, it is practically preferable due to its numerical stability. 

We empirically analyze the advantages of PD neural estimation on three applications. First, we cast the challenging MI estimation problem to be a PD estimation problem. The re-formulation bypasses calculating MI lower bounds in prior work~\cite{belghazi2018mine,poole2019variational}, which suffers from large variance~\cite{song2019understanding} in practice. Our empirical results demonstrate the low variance and bias of the proposed approach when comparing to prior MI neural estimators. Second, our PD estimation objectives also inspire new losses for contrastive self-supervised representation learning. Surprisingly, {\em Density-Ratio Fitting} inspired loss results in a consistent improvement over prior work in both shallow~\cite{tschannen2019mutual} and deep~\cite{bachman2019learning} neural architectures. Third, we study the use of PD estimation for data containing information across modalities. More specifically, we analyze the cross-modal retrieval task on human speech and text corpora. We make our experiments publicly available at \Code.

\section{Related Work}
\paragraph{Point-wise Dependency Estimation} Prior literature studies point-wise dependency (PD) with three groups of estimation methods: {\em counting-based}~\cite{church1990word,bouma2009normalized,levy2014neural}, {\em kernel-based}~\cite{yokoi2018pointwise}, and {\em likelihood-based}~\cite{li2015diversity}. {\em Counting-based} methods approximate the joint density by counting the occurrence of the pair (i.e., $(x,y)$) and the marginal density by counting the presence of the individual outcome (i.e., $x$ or $y$). Counting based approaches can only work on discrete data and may be unrealistic when the data is sparse. {\em Kernel-based} method, particularly pointwise HSIC~\cite{yokoi2018pointwise}, can be seen as a smoothed variant of the counting-based methods, which adopts the kernel to measure the similarity between sparse data. Although this method manifests nice robustness to sparse data, its computational cost is high with high-dimensional data. {\em Likelihood-based} approaches instead approximate conditional likelihood (i.e., $p(y|x)$) and marginal likelihood (i.e., $p(y)$) using function approximators such as neural networks. Although this approach can be adapted to continuous data, it involves marginal likelihood estimation, which is challenging~\cite{kingma2013auto,goodfellow2014generative} and may perform poorly in practice. On the other hand, our presented approaches involve no marginal likelihood estimation, can work on both discrete and continuous data, and leverage neural networks with gradient descent in high-dimensional settings.

\paragraph{Density Ratio Estimation} To calculate the ratio between densities ($p(x)/q(x)$), prior density ratio estimation approaches~\cite{sugiyama2012density3,sugiyama2012density} propose to estimate the ratio directly and avoid estimating the density ($p(x)$ and $q(x)$). For example, Sugiyama {\em et al.}~\cite{sugiyama2012density3} fit the true density ratio model under the Bregman divergence~\cite{bregman1967relaxation} and further develop a robust density estimation method under the power divergence~\cite{basu1998robust}. While it is straightforward to apply these approaches to PD estimation, these approaches are studied in the context of kernel-based methods, which can make it difficult to apply in practice when data is high-dimensional and complex-structured. Our approaches contrarily take advantage of high-capacity neural networks.

\paragraph{Neural Methods for Mutual Information Estimation} Recent approaches~\cite{belghazi2018mine,poole2019variational} present neural methods that estimate mutual information (MI) via its variational bounds. They consider MI 1) lower bounds such as Donsker-Varadhan bound~\cite{donsker1983asymptotic} and Nguyen-Wainwright-Jordan bound~\cite{nguyen2010estimating}; and 2) upper bound such as Barber-Agakov bound~\cite{barber2003algorithm}. These bounds exhibit inevitable large variance~\cite{song2019understanding} and have severe training instability in practice~\cite{hjelm2018learning,tschannen2019mutual}. In our discussion, we show that we can obtain PD when optimizing these bounds. Additionally, we present alternative PD estimation methods that do not involve calculating MI variational bounds and are favorable in practice.

\section{Point-wise Dependency Neural Estimation}

Our paper aims to identify the association for a pair of outcomes $(x, y)\in\xxspace\times\yyspace$ by studying their point-wise dependency. We use an uppercase letter to denote a random variable (i.e., $X$), a lowercase letter to indicate an outcome $x$ drawn from a particular distribution (i.e., $x\sim P_X$), and a calligraphy letter $\xxspace$ to represent a sample space (i.e., $x \in \xxspace$). The joint distribution of $X,Y$ is represented by $P_{X,Y}$, and the product of their marginals is represented by $P_{X}P_{Y}$. Throughout the paper, we use the conventional notation $I(X;Y)$ to denote the mutual information between random variables $X$ and $Y$.

Formally, we define the following point-wise dependency (PD) to quantitatively measure the discrepancy between {\em the probability of their co-occurrence} and {\em the probability of independent occurrences}. 

\begin{defn}[Point-wise Dependency]
Given a pair of outcomes $(x,y)\sim P_{X,Y}$, their point-wise dependency is defined as $r(x,y)\defeq p(x,y)/p(x)p(y)$.
\label{defn:PD}
\end{defn}
PD is non-negative. Intuitively, when $r(x,y) > 1$, it means $(x, y)$ co-occur more often than their independent occurances. Similarly, when $r(x, y) \leq 1$, it means they co-occur less frequently. Our goal is to estimate $r(x,y)$ by approximating it using neural network $\hat{r}_{\theta}(x, y)$ with parameter $\theta\in\Theta$. 

\subsection{Mutual Information and Point-wise Dependency}
\label{subsec:pmi_esti}
A related quantitative measurement of point-wise dependency is Point-wise mutual information (PMI)~\cite{bouma2009normalized}, which is the logarithm of PD (PMI~$\defeq f(x,y) = \log r(x, y)$). 
In this subsection, we shall discuss parametrized estimation of PMI using neural networks $\hat{f}_\theta (x,y)$ with parameter $\theta$. By definition, mutual information $I(X;Y)$ is the expected value of PMI: $I(X;Y) = \Exp_{P}[{\rm log}\,r(X,Y)] = \Exp_{P}[f(X,Y)]$. Hence by using $\hat{f}_\theta$ as a plug-in, we can obtain an approximation of the mutual information with $\Exp_P[\hat{f}_\theta(X, Y)]$. Reversely, we will show that PMI can be obtained when optimizing MI (neural) variational bounds and present two methods to do so, one as unconstrained optimization and the other as constrained optimization problem.

\paragraph{(Unconstrained Optimization) Variational Bounds of Mutual Information}
Recent work~\cite{belghazi2018mine,poole2019variational} proposes to estimate MI using neural networks by exploiting either the variational MI lower bounds~\cite{belghazi2018mine} or the variational MI form~\cite{poole2019variational}. In particular, Belghazi {\em et al.}~\cite{belghazi2018mine} proposed the $I_{\rm DV}$ estimator, standing for Donsker-Varadhan (DV) lower bound~\cite{donsker1983asymptotic} of MI. On the other hand, Poole {\em et al.}~\cite{poole2019variational} proposed the $I_{\rm JS}$ estimator, corresponding to using f-GAN objective~\cite{nowozin2016f} as a lower bound of Jensen-Shannon (JS) divergence between $P_{X,Y}$ and $P_{X}P_{Y}$. $I_{\rm JS}$ is found to be more stable then $I_{\rm DV}$ and other variational lower bounds, and thus it is widely used in prior work~\cite{hjelm2018learning,poole2019variational,song2019understanding}, defined as follows:
\begin{equation}
I_{\rm JS} \defeq \underset{\theta \in \Theta}{\rm sup}
\,\mathbb{E}_{P_{X,Y}}\Big[-{\rm softplus} \Big(-\hat{f}_\theta(x,y)\Big)\Big] - \mathbb{E}_{P_{X} P_{Y}}\Big[{\rm softplus} \Big(\hat{f}_\theta(x,y)\Big)\Big],
\label{eq:I_JS}
\end{equation}
where we use ${\rm softplus}$ to denote ${\rm softplus}\,(x) = {\rm log}\,\left(1+{\rm exp}\,(x)\right)$. It could be readily verified that the optimal $\hat{f}_\theta^*(x,y) = \log\left(p(x,y)/p(x)p(y)\right)$~\cite{poole2019variational}. We refer this objective as {\em Variational Bounds of Mutual Information} approach for PMI estimation. 


\paragraph{(Constrained Optimization) Density Matching}
This method considers to match the true joint density $p(x,y)$ and the estimated joint density $\hat{p}_\theta (x,y)  \defeq e^{\hat{f}_\theta (x,y)} p(x) p(y)$ by minimizing the following KL divergence:
\begin{equation*}
\underset{\theta \in \Theta}{\rm inf}\,\,\kldiv{P_{X,Y}}{\hat{P}_{\theta X,Y}} \defeq \underset{\theta \in \Theta}{\rm inf}\,\,I(X;Y) - 
 \mathbb{E}_{P_{X,Y}}\Big[\hat{f}_\theta(x,y) \Big] \Leftrightarrow\underset{\theta \in \Theta}{\rm sup}\,\,
 \mathbb{E}_{P_{X,Y}}\Big[\hat{f}_\theta(x,y) \Big].
\end{equation*}
Since KL divergence has a minimum value of $0$, it is easy to see that $\forall\theta\in\Theta$, $\mathbb{E}_{P_{X,Y}}[\hat{f}_\theta(x,y)]$ is a lower bound of MI. Note that this objective is a constrained optimization problem, since we need to ensure the estimated joint density is a valid density function: $\hat{p}_\theta (x,y)\geq 0$ and $\iint \hat{p}_\theta (x,y)~{\rm d}x{\rm d}y = 1$. Equivalently, the constraints could be formed as $e^{\hat{f}_\theta (x,y)}\geq 0$ (trivially true) and $\mathbb{E}_{P_XP_Y}[e^{\hat{f}_\theta(x,y)}] = 1$. Putting everything together, we can reformulate the following constrained optimization problem:
\begin{equation*}
\underset{\theta \in \Theta }{\max}\,\, \mathbb{E}_{P_{X,Y}}[\hat{f}_\theta (x,y)],\quad{\text{subject to}}\,\,\mathbb{E}_{P_X  P_Y}[e^{\hat{f}_\theta(x,y)}]=1 ,
\end{equation*}
which is also called KL importance estimation procedure~\cite{sugiyama2008direct} with a unique solution $\hat{f}^*_\theta (x,y)= \log\left(p(x,y)/p(x)p(y)\right)$. The Lagrangian of the above constrained problem is
\begin{equation}
\small
\begin{split}
\underset{\theta \in \Theta }{\max}\,\, \mathbb{E}_{P_{X,Y}}[\hat{f}_\theta (x,y)] - \lambda \cdot\left( \mathbb{E}_{P_X P_Y}[e^{\hat{f}_\theta(x,y)}] - 1 \right), 
\end{split}
\label{eq:DM1}
\end{equation}
where $\lambda\in\RR$ is the dual variable. Furthermore, penalty method could also be used to transform the original constrained optimization problem to an unconstrained one:
\begin{equation}
\small
\begin{split}
\underset{\theta \in \Theta}{\max}\,\, \mathbb{E}_{P_{X,Y}}[\hat{f}_\theta (x,y)] - \eta\cdot\left({\rm log}\, \mathbb{E}_{P_X P_Y}[
e^{\hat{f}_\theta(x,y)}]\right)^2,
\end{split}
\label{eq:DM2}
\end{equation}
where $\eta > 0$ is the penalty coefficient. We refer Eq.~\eqref{eq:DM1} as {\em Density Matching I} and Eq.~\eqref{eq:DM2} as {\em Density Matching II} for PMI estimation.

\subsection{Proposed Methods for Point-wise Dependency (PD) Estimation}
In the last section, we introduce how to obtain PMI by optimizing various MI variational bounds. In this section, instead of estimating PMI, we present two methods to estimate PD ($p(x,y)/p(x)p(y)$), i.e., the {\em Probabilistic Classifier} method and the {\em Density-Ratio Fitting} method. We argue that the presented PD estimation methods admit better training stability than the PMI estimation methods discussed in the last section. On the one hand, the Probabilistic Classifier method casts PD estimation as a binary classification task, where the binary cross-entropy loss can be used and optimized in existing optimization packages~\cite{paszke2019pytorch,abadi2016tensorflow}. On the other hand, the Density-Ratio Fitting method contains no logarithm or exponentiation, which are often the roots of the instability in MI (or PMI) estimation~\cite{poole2019variational,song2019understanding}. In what follows, we present both methods in a sequel.

\paragraph{Probabilistic Classifier Method}
This approach casts the PD estimation as the problem of estimating the `class'-posterior probability. First, we use a Bernoulli random variable $C$ to classify the samples drawn from the joint density ($C=1$ for $(x,y)\sim P_{X,Y}$) and the samples drawn from product of the marginal densities ($C=0$ for $(x,y)\sim P_{X} P_{Y}$). Equivalently, the likelihood function $p(x, y\mid C=1)\defeq p(x,y)$ and $p(x, y\mid C=0)\defeq p(x)p(y)$.
By Bayes' Theorem, we re-express PD by the ratio of two class-posterior probability:
\begin{equation*}
r(x,y)  = \frac{p(x,y)}{p(x)p(y)}=\frac{p(x, y\mid C = 1)}{p(x, y\mid C = 0)}  = \frac{p(C=0)}{p(C=1)}\frac{p(C=1\mid x,y)}{p(C=0\mid x,y)}.
\end{equation*}
In the above equation, the ratio $\frac{p(C=0)}{p(C=1)}$ can be approximated by the ratio of the sample size: 
\begin{equation*}
\frac{\hat{p}(C=0)}{\hat{p}(C=1)}  = \frac{(n_{P_X P_Y}) / (n_{P_X  P_Y} + n_{P_{X,Y}})}{(n_{P_{X,Y}}) / (n_{P_X  P_Y} + n_{P_{X,Y}})}  = \frac{n_{P_X  P_Y}}{n_{P_{X,Y}}},
\end{equation*}
and we use a probability classifier $\hat{p}_\theta (C\mid x,y)$ parameterized by a neural network $\theta$ to approximate the class-posterior classifier $p(C\mid x,y)$. By adopting the binary cross-entropy loss, the objective
has the following form:
\begin{equation}
\underset{\theta \in \Theta}{\max}\,\,  \mathbb{E}_{P_{X,Y}}[{\log}\,\hat{p}_\theta (C=1\mid x,y)]  + \mathbb{E}_{P_{X}P_{Y}}[{\log}\,(1-\hat{p}_\theta (C=1\mid x,y))].
\label{eq:PC_loss}
\end{equation}
Then, bringing all the equations together, we obtain the {\em Probabilistic Classifier} PD estimator:
\begin{equation}
\hat{r}_\theta (x,y)  = \frac{n_{P_XP_Y}}{n_{P_{X,Y}}} \frac{\hat{p}_\theta (C=1\mid x,y)}{\hat{p}_\theta(C=0\mid x,y)},\quad{\rm with}\,\,(x,y)\sim P_{X,Y}\,\,{\rm or}\,\,(x,y)\sim P_{X}P_{Y}.
\label{eq:PC_all}
\end{equation}

\paragraph{Density-Ratio Fitting Method}
This approach considers to minimize the expected least-square difference between the true PD $r(x,y)$ and the estimated PD $\hat{r}_\theta(x,y)$:
\begin{equation}
\underset{\theta \in \Theta}{\rm inf}\,\,\mathbb{E}_{P_X  P_Y} [\big(r(x,y) - \hat{r}_\theta(x,y)\big)^2] \Leftrightarrow \underset{\theta \in \Theta}{\rm sup}\,\,\mathbb{E}_{P_{X,Y}}[\hat{r}_\theta(x,y)] - \frac{1}{2}\mathbb{E}_{P_X P_Y}[\hat{r}_\theta^2(x,y)].
\label{eq:DRF}
\end{equation}
The objective is also called least-square density-ratio fitting method~\cite{kanamori2009least} and has a unique solution $\hat{r}^*_\theta (x,y)= p(x,y)/p(x)p(y)$. We refer Eq.~\eqref{eq:DRF} as {\em Density-Ratio Fitting} PD estimation.

\section{Application I: Mutual Information Estimation}
By definition, as the average effect of point-wise dependency (PD), Mutual Information (MI) measures the statistical independence between random variables:
\begin{equation}
\begin{split}
I(X;Y) & = \kldiv{P_{X,Y}}{P_X P_Y}
= \iint p(x,y) {\rm log}\, \frac{p(x,y)}{p(x)p(y)} \,{\rm d}x{\rm d}y = \mathbb{E}_{P_{X,Y}}[{\rm log}\,r(x,y)] \\ 
& \approx \mathbb{E}_{ P_{X,Y}}[{\rm log}\, \hat{r}_\theta(x,y)]\approx \mathbb{E}_{P_{X,Y}}[ \hat{f}_\theta(x,y)],
\end{split}
\label{eq:MI_plugging_in} 
\end{equation}
where we estimate MI by directly plugging-in PD (i.e., $\hat{r}_\theta$ in Eq.~\eqref{eq:PC_all},~\eqref{eq:DRF}) or PMI (i.e., $\hat{f}_\theta$ in Eq.~\eqref{eq:I_JS},~\eqref{eq:DM1}, and~\eqref{eq:DM2}). In summary, we cast the MI estimation problem to a PD or PMI estimation problem.

\paragraph{Baseline Models} 
Instead of approximating MI by plugging-in the estimated PD or PMI, prior work focuses on establishing tractable and scalable bounds for MI~\cite{oord2018representation,belghazi2018mine,poole2019variational,song2019understanding}, in which the bounds can be computed via gradient descent over neural networks. Strong baselines include CPC~\cite{oord2018representation}, NWJ~\cite{belghazi2018mine}, JS~\cite{poole2019variational}, DV (MINE)~\cite{belghazi2018mine}, and SMILE~\cite{song2019understanding}. To understand the differences, we separate MI neural estimation methods into two procedures: {\em learning} and {\em inference}. The learning step learns the parameters when estimating 1) point-wise dependency/ logarithm of point-wise dependency; or 2) MI lower bound. The inference step considers the parameters from the learning step and infers value for 1) MI itself; or 2) a lower bound of MI. We summarize different approaches in Table 1. For completeness, one may see Supplementary for more details about these bounds.



\begin{table}[t!]
\caption{\small MI neural estimation methods. The estimation procedure is dissected into learning and inference phases, which may use different objectives. Baselines consider to estimate MI via lower bounds, while ours consider to estimate MI via plugging in PD ($\hat{r}_\theta$) or PMI ($\hat{f}_\theta$) estimators. 
}
\vspace{-1mm}
\parbox{.42\linewidth}{
\centering
\scalebox{0.74}{
\begin{tabular}{lcc}
\toprule
Baselines          & Learning & Inference \\
\midrule
CPC~\cite{oord2018representation}             &     $I_{\rm CPC}$~\cite{oord2018representation}                & $I_{\rm CPC}$~\cite{oord2018representation}                        \\
NWJ~\cite{belghazi2018mine} &  $I_{\rm NWJ}$~\cite{nguyen2010estimating,belghazi2018mine}                  &  $I_{\rm NWJ}$~\cite{nguyen2010estimating,belghazi2018mine} \\
JS~\cite{poole2019variational}              &   $I_{\rm JS}$~\cite{nowozin2016f} (Eq.~\eqref{eq:I_JS})                  &   $I_{\rm NWJ}$~\cite{nguyen2010estimating,belghazi2018mine} \\
DV (MINE)~\cite{belghazi2018mine}       &          $I_{{\rm DV}}$~\cite{belghazi2018mine}          &    $I_{\rm DV}$~\cite{donsker1983asymptotic,belghazi2018mine}                              \\
SMILE~\cite{song2019understanding}          &    $I_{\rm JS}$~\cite{nowozin2016f} (Eq.~\eqref{eq:I_JS})                & $I_{{\rm DV}}$~\cite{donsker1983asymptotic,belghazi2018mine} \\\bottomrule     
\end{tabular}}
}
\parbox{.55\linewidth}{
\centering
\scalebox{0.69}{
\begin{tabular}{lcc}
\toprule
Ours          & Learning & Inference \\
\midrule
Variational MI Bounds             &    I$_{\rm JS}$~\cite{nowozin2016f} (Eq.~\eqref{eq:I_JS})                &   Eq.~\eqref{eq:MI_plugging_in} with $\hat{f}_\theta$                    \\
Probabilistic Classifier &  Eq.~\eqref{eq:PC_loss}              &  Eq.~\eqref{eq:MI_plugging_in} with $\hat{r}_\theta$ in Eq.~\eqref{eq:PC_all}                      \\
Density Matching I             &   Eq.~\eqref{eq:DM1}                   &   Eq.~\eqref{eq:MI_plugging_in}  with $\hat{f}_\theta$                        \\
Density Matching II      &          Eq.~\eqref{eq:DM2}           &    Eq.~\eqref{eq:MI_plugging_in}  with $\hat{f}_\theta$ \\
Density-Ratio Fitting          &    Eq.~\eqref{eq:DRF}                 & Eq.~\eqref{eq:MI_plugging_in} with $\hat{r}_\theta$ \\
\bottomrule     
\end{tabular}}
}
\label{tbl:estimators}
\end{table}
\begin{figure}[t!]
\includegraphics[width=\textwidth]{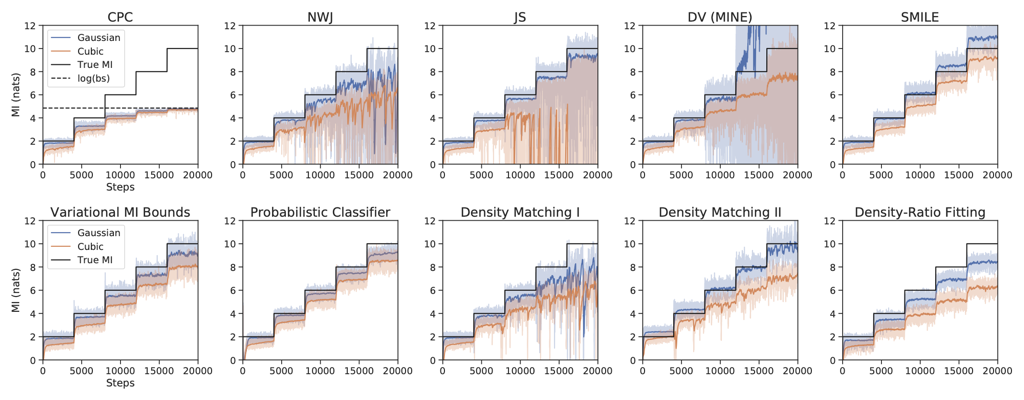}
\vspace{-7mm}
\caption{\small {\bf Gaussian} and {\bf Cubic} task for correlated Guassians with tractable ground truth MI. The upper row are the baselines and the lower row are our methods. Network, learning rate, optimizer, and batch size are fixed for all MI neural estimators. The only differences are the learning and inference objectives shown in Table~\ref{tbl:estimators}.}
\label{fig:gaussian_exp}
\vspace{-2mm}
\end{figure}

\paragraph{Benchmarking on Correlated Gaussians} To evaluate the performance between different MI neural estimators, we consider the standard tasks on correlated Gaussians~\cite{belghazi2018mine,poole2019variational,song2019understanding}. In particular, we draw $(x,y)$ from two $20$-dimensional Gaussians with correlation $\rho$, which is referred as {\bf Gaussian} task. Then, we apply a cubic transformation on $y$ so that $y\mapsto y^3$, which is referred to as {\bf Cubic} task. These two tasks have tractable ground truth MI $=-10\,{\rm log}\,(1-\rho^2)$. We train all models for $20,000$ iterations, starting from MI $=2$ and increasing it by $2$ per $4,000$ iterations. We fix the network, learning rate, optimizer, and batch size across all the estimators for a fair comparison. The only differences are the objectives considered in the learning and inference in MI estimation (shown in Table~\ref{tbl:estimators}). 

\paragraph{Results \& Discussions} We present the results in Figure~\ref{fig:gaussian_exp} and leave more training details in Supplementary. In the following, we discuss bias-variance trade-offs for different approaches. We first discuss general observations. Most of the estimators have both larger bias and variance with larger ground truth MI. The only exception is CPC~\cite{oord2018representation}, where its value is upper bounded by ${\rm log\,(batch\_size)}$~\cite{poole2019variational}. The bias is also larger in {\bf Gaussian} task than in {\bf Cubic} task except for DV~\cite{belghazi2018mine}. Next, we discuss the differences among estimators in detail. CPC~\cite{oord2018representation} has the smallest variance, yet it is highly biased. Although having larger variance than CPC, SMILE~\cite{song2019understanding}/ Variational MI Bounds/ Probabilistic Classifier/ Density Matching I \& II/ Density-Ratio Fitting approaches have a much lower bias. Among them, Probabilistic Classifier and Density-Ratio Fitting approaches have the smallest variance. NWJ~\cite{belghazi2018mine}/ JS~\cite{poole2019variational}/ DV~\cite{belghazi2018mine}, whereas, have both large variance and bias. Note that JS~\cite{poole2019variational} has larger variance is because using $I_{\rm NWJ}$ objective during inference. To sum up, we see that the plug-in MI estimators enjoy smaller variance and bias when comparing to most of the lower bound methods. 

\paragraph{Theoretical Analysis}
In Eq.~\eqref{eq:MI_plugging_in}, we present a high-level intuition that a good estimation of the PD function $\hat{r}_\theta(x,y)$ could be used to estimate the mutual information. In what follows, we present a formal justification for this argument. To begin with, let $P_{X,Y}^{(n)}$ denote the empirical distribution of the ground-truth joint distribution $P_{X,Y}$ estimated from $n$ samples drawn uniformly at random from $P_{X,Y}$. Then our estimator of the mutual information is given by $\widehat{I}_\theta^{(n)}(X;Y)\defeq \Exp_{P^{(n)}_{X,Y}}[\log \hat{r}_\theta(x, y)]$.

At a high level, our arguments contain two parts. In the first part, we show that w.h.p. (with high probability)\ $\widehat{I}_\theta^{(n)}(X;Y)$ is close to $\Exp_{P_{X,Y}}[\log \hat{r}_\theta(x, y)]$. In the second part, we apply the universal approximation lemma of neural networks~\citep{hornik1989multilayer} to show that there exists $\hat{r}_\theta(\cdot, \cdot)$ that is close to $r(\cdot,\cdot)$. Formally, let $\FF\defeq \{\hat{r}_\theta: \theta\in\Theta\subseteq\RR^d\}$ be the set of neural networks where the parameter $\theta$ is a $d$-dimensional vector. Throughout the analysis, we assume the following assumptions:
\begin{assumption}[Boundedness of the density ratio]
There exist universal constants $C_l \leq C_u$ such that $\forall \hat{r}_\theta\in\FF$ and $\forall x,y$, $C_l\leq \log\hat{r}_\theta(x, y)\leq C_u$.
\label{assu:bound}
\end{assumption}
\begin{assumption}[$\log$-smoothness of the density ratio]
There exists $\rho > 0$ such that for $\forall x,y$ and $\forall \theta_1, \theta_2\in\Theta$, $|\log\hat{r}_{\theta_1}(x, y) - \log\hat{r}_{\theta_2}(x, y)|\leq \rho\cdot \|\theta_1 - \theta_2\|$.
\label{assu:smooth}
\end{assumption}
Assumption~\ref{assu:bound} basically asks the output of a neural net to be bounded and Assumption~\ref{assu:smooth} says that for any given input pair, the output of the network should only change slightly if we just slightly perturb the network weights. Both assumptions are mostly verified in practical networks. Based on these two assumptions, the following lemma is adapted from~\citet{bartlett1998sample} that bounds the rate of uniform convergence of a function class in terms of its covering number. The original lemma is based on the $L_\infty$ norm of the function class; whereas the following one, we use the $L_2$ norm on $\Theta$.
\begin{restatable}{lemma}{covers}
(estimation). Let $\eps > 0$ and $\covering(\Theta, \eps)$ be the covering number of $\Theta$ with radius $\eps$ under $L_2$ norm. Let $P_{X,Y}$ be any distribution where $S = \{x_i, y_i\}_{i=1}^n$ are sampled from and define $M\defeq C_u - C_l$, then 
\begin{equation}
    \Pr_S\left(\sup_{\hat{r}_\theta\in\FF} \left|\widehat{I}_\theta^{(n)}(X;Y) - \Exp_{P_{X,Y}}[\log \hat{r}_\theta(x, y)]\right| \geq \eps\right) \leq 2\mathcal{N}(\Theta, \eps/ 4\rho)\exp\left(-\frac{n\eps^2}{2M^2}\right).
\end{equation}
\label{lemma:1}
\end{restatable}
Next lemma is derived from~\citep{hornik1989multilayer}, which shows that neural networks are universal approximators:
\begin{lemma}[\citet{hornik1989multilayer}, approximation]
Let $\eps > 0$. There exists $d \in\Nat$ and a family of neural networks $\FF\defeq \{\hat{r}_\theta: \theta\in\Theta\subseteq\RR^d\}$ where $\Theta$ is compact, such that $\inf_{ \hat{r}_\theta\in\FF}\left|\Exp_{P_{X, Y}}[\log \hat{r}_\theta(x, y)] - I(X; Y)\right|\leq \eps$.
\label{lemma:2}
\end{lemma}
Combining both lemmas, we are ready to state the following main result:
\begin{restatable}{theorem}{mainthm}
Let $0 < \delta < 1$. There exists $d \in\Nat$ and a family of neural networks $\FF\defeq \{\hat{r}_\theta: \theta\in\Theta\subseteq\RR^d\}$ where $\Theta$ is compact, so that $\exists \theta^*\in\Theta$, with probability at least $1 - \delta$ over the draw of $S = \{x_i, y_i\}_{i=1}^n\sim P_{X,Y}^{\otimes n}$, 
\begin{equation}
    \left|\widehat{I}_{\theta^*}^{(n)}(X;Y) - I(X; Y)\right| \leq O\left(\sqrt{\frac{d + \log(1/\delta)}{n}}\right).
\end{equation}
\end{restatable}
It is worth pointing out that the above theorem is a theorem of existence, but \emph{not} a constructive theorem, meaning that it does not give an estimator explicitly. To sum up, it shows that there exists a neural network $\theta^*$ such that, w.h.p., $\widehat{I}_{\theta^*}^{(n)}(X;Y)$ can approximate $I(X;Y)$ with $n$ samples at a rate of $O(1/\sqrt{n})$.

\section{Application II: Self-supervised Representation Learning}
\label{sec:SSRL}

Self-supervised representation learning aims at extracting task-relevant information without access to label or downstream signals. Among different self-supervised representation learning techniques, {\em contrastive learning} may be the most popular one with empirical~\cite{agrawal2015learning,arandjelovic2017look,jayaraman2015learning,oord2018representation,bachman2019learning,chen2020simple,tian2019contrastive,hjelm2018learning,he2019momentum,kong2019mutual,ozair2019wasserstein,henaff2019data} and theoretical~\cite{arora2019theoretical,tsai2020demystifying} support. The core of contrastive learning is having the representations sampled from similar pairs be differentiated from random pairs. In other words, we hope that the representations learned from the similar pairs have higher point-wise dependency than the random pairs. Let $v_1$/$v_2$ denote two different views for the same data,  $v_2'$ represent a view from a different data, and $F$/$G$ be two mapping functions from data to representations. In short, contrastive learning objective learns $F$/$G$ such that $r(F(v_1),G(v_2))$ is much larger than $r(F(v_1),G(v_2'))$.

\paragraph{Connection between Contrastive Learning and PD}  Our goal is to show that our learning objectives resemble contrastive learning. We first take the {\em Probabilistic Classifier} approach as an example and incorporate the learning of $F/G$, which we name it as {\em Probabilistic Classifier Coding} (PCC):
\begin{equation}
\small
\underset{F,G}{\rm sup}\,\,\underset{\theta \in \Theta}{\rm sup}\,\,  \mathbb{E}_{P_{\mathcal{V}_1,\mathcal{V}_2}}[{\rm log}\,\hat{p}_\theta (c=1|(F(v_1),G(v_2)))]  + \mathbb{E}_{P_{\mathcal{V}_1}P_{\mathcal{V}_2}}[{\rm log}\,\Big(1-\hat{p}_\theta (c=1|(F(v_1),G(v_2')))\Big)],
\label{eq:PCC}
\end{equation}
which aims at learning $F/G$ to better classify (i.e., differentiate) between similar or random data pairs. Next, we consider the {\em Density-Ratio Fitting} approach, which we refer to the objective as {\em Density-Ratio Fitting Coding} (D-RFC):
\begin{equation}
\small
\underset{F,G}{\rm sup}\,\,\underset{\theta \in \Theta}{\rm sup}\,\,\mathbb{E}_{P_{\mathcal{V}_1,\mathcal{V}_2}}[\hat{r}_\theta(F(v_1),G(v_2))] - \frac{1}{2}\mathbb{E}_{P_{\mathcal{V}_1} P_{\mathcal{V}_2}}[\hat{r}_\theta^2(F(v_1),G(v_2'))],
\label{eq:DRFC}
\end{equation}
which aims at learning $F$/$G$ to maximize $\hat{r}_\theta(F(v_1),G(v_2))$ and minimize $\hat{r}_\theta(F(v_1),G(v_2'))$. We leave the discussion for the adaptations of Variational MI Bounds, Density Matching I ,and Density Matching II in Supplementary.

\paragraph{Baseline Model} The most adopted contrastive representation learning objective is Contrastive Predictive Coding (CPC)~\cite{oord2018representation}:
\begin{equation*}
\begin{split}
 \underset{F,G}{\rm sup}\,\,\underset{\theta \in \Theta}{\rm sup}\,\,  \mathbb{E}_{(v_1^1,v_2^1)\sim P_{\mathcal{V}_1,\mathcal{V}_2}, \cdots (v_1^n,v_2^n)\sim P_{\mathcal{V}_1,\mathcal{V}_2}}[\frac{1}{n}\sum_{i=1}^{n}{\rm log}\,\frac{e^{\hat{c}_\theta(F(v_1^i), G(v_2^i))}}{\frac{1}{n}\sum_{j=1}^{n}e^{\hat{c}_\theta(F(v_1^i), G(v_2^j))}}],
\end{split}
\end{equation*} 
where $\{v_1^i, v_2^i\}_{i=1}^n$ are independently and identically sampled from $P_{\mathcal{V}_1,\mathcal{V}_2}$. $\hat{c}_\theta (\cdot)$ is a function that takes the representations learned from the data pairs and returns a scalar.

\begin{figure}[t]
  \centering
  \includegraphics[width=\linewidth]{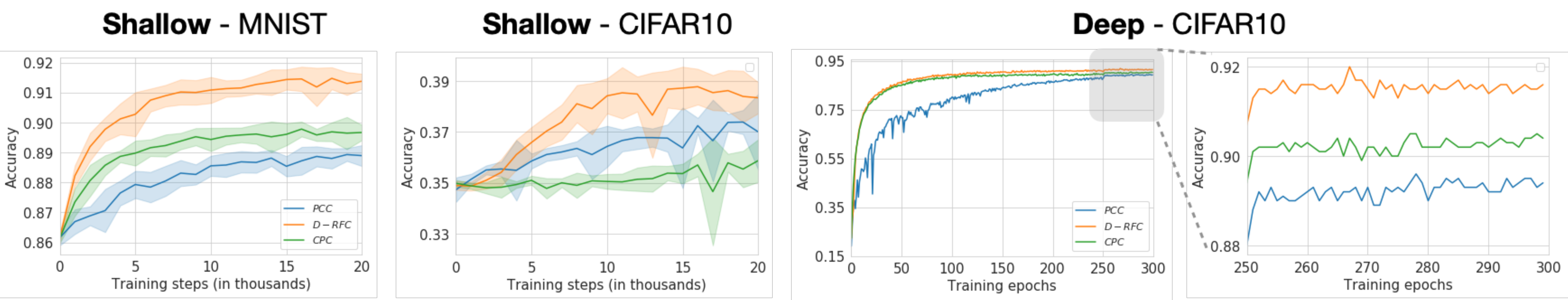}
  \vspace{-5mm}
\caption{\small  {\bf Shallow}~\cite{tschannen2019mutual} and {\bf Deep}~\cite{bachman2019learning} task for self-supervised visual representation learning using {\em downstream linear evaluation protocol}. 
We compare the presented Probabilistic Classifier Coding (PCC) and Density-Ratio Fitting Coding (D-RFC) with baseline Contrastive Predictive Coding (CPC). Network, learning rate, optimizer, and batch size are fixed for all the methods. The only differences are the learning objectives.
}
\label{fig:RL}
\vspace{-3mm}
\end{figure}

\paragraph{Experimental Setup} We compare our proposed approaches with CPC~\cite{oord2018representation} on two tasks~\cite{tschannen2019mutual,bachman2019learning}. Due to the fact that the performance of the self-supervisedly learned representations strongly depends on the choice of feature extractor architectures and the parametrization of the employed MI estimators~\cite{tschannen2019mutual}. For a fair comparison, we fix the network, learning rate, optimizer, and batch size when comparing between different objectives. In the first set of experiments, we choose a relatively shallow network as suggested by Tschannen {\em et al.}~\cite{tschannen2019mutual}, performing self-supervised learning experiments on MNIST~\cite{lecun1998gradient} and CIFAR10~\cite{krizhevsky2009learning}. We report the average and standard deviations from $10$ random trials. This task is referred to as {\bf shallow} experiment. In the second set of experiments, we choose a relatively deep network as suggested by Bachman {\em et al.}~\cite{bachman2019learning}, performing experiments on CIFAR10. This task is referred to as {\bf deep} experiment. Both the {\bf shallow} and {\bf deep} tasks perform representation learning without access to the label information, and then the performance is evaluated by {\em downstream linear evaluation protocol}~\cite{oord2018representation,henaff2019data,tian2019contrastive,hjelm2018learning,bachman2019learning,tschannen2019mutual,kolesnikov2019revisiting}. Specifically, a linear classifier is trained from the self-supervisedly learned (fixed) representation to the labels on the training set. We present the results with convergence in Figure~\ref{fig:RL}. One may see Supplementary for more details.

\paragraph{Results \& Discussions} Prior approaches~\cite{tschannen2019mutual,poole2019variational,song2019understanding,ozair2019wasserstein} contend that a valid MI lower bound or an objective with better MI estimation may not result in better representations. We have a similar observation that D-RFC performs the best (when comparing to CPC and PCC) while it is neither a lower bound of MI nor the best objective of MI estimation. Next, we see an inconsistent trend when comparing PCC to CPC. In the Shallow task on CIFAR10, PCC performs better than CPC, while it performs worse on the other experiments. To sum up, we show our PD estimation objectives can be used for self-supervised representation learning, which is either at par or better than prior approaches.

\section{Application III: Cross-modal Learning}

\begin{table}[t!]
\caption{\small 
  Cross-modal Retrieval task with  unsupervised word features across acoustic and textual modalities. {\em Probabilistic Classifier} approach is used to estimate PD between the audio and textual features of a given word. The estimator is trained on the training split. We report the $1:5$ matching results from audio to textual features on the test split, where we obtain 96.24\% top-1 retrieval accuracy. 
}
\vspace{-1mm}
\centering
\scalebox{0.77}{
\begin{tabular}{c|ccccc}
\toprule
\multicolumn{6}{c}{Correct Audio-Textual Retrieval Examples (Top-1 Accuracy: 96.24\%)}   \\ 
\midrule \midrule Audio Feature & \multicolumn{5}{c}{Textual Features (Ranked by logarithm of point-wise dependency)}                                                                      \\ \midrule
depths        & \textbf{depths (15.22)}     & mildewed (-58.62)          & lugged (-92.24)   & alison (-108.02)   & raffleshurst (-161.74) \\
receptacle    & \textbf{receptacle (1.32)} & bloated (-15.41)           & recreate (-39.77)  & sting (-90.51)     & pity (-104.44)         \\
frontiers     & \textbf{frontiers (3.36)}   & institution (-31.01)       & laterally (-54.17) & pretends (-105.11) & vibrating (-124.88)    \\ \bottomrule
\multicolumn{6}{c}{} \\
\toprule
\multicolumn{6}{c}{Incorrect Audio-Textual Retrieval Examples}   \\ 
\midrule \midrule Audio Feature & \multicolumn{5}{c}{Textual Features (Ranked by logarithm of point-wise dependency)}                                                                      \\ \midrule
cos           & tortoise (-2.33)            & \textbf{cos (-10.72)}      & tickling (-12.53)  & undressed (-18.11) & cromwell's (-44.31)    \\
elbowing      & itinerary (-6.51)           & \textbf{elbowing (-8.22)} & swims (-12.98)     & rigid (-24.14)     & integrity (-39.76) \\
alma's      & roughness (-3.11)           & \textbf{alma's (-3.67)} & montreal (-11.81)     & tuneful (-12.22)     & levant (-18.26) \\
\bottomrule
\end{tabular}}
\label{tbl:CM}
\end{table}

In this section, we discuss the usage of point-wise dependency (PD) estimation for data containing information across modalities - audio and text.

\paragraph{Experimental Setup - Cross-modal Retrieval} We instantiate the discussion using unsupervised word features\footnote{The word features can be downloaded from~\url{https://github.com/iamyuanchung/speech2vec-pretrained-vectors}.} which are learned from text corpora (i.e., Word2Vec~\cite{mikolov2013efficient} method) and human speech (i.e., Speech2Vec~\cite{chung2018speech2vec} method). In particular, in this dataset, a word feature has two distinct features: audio and textual feature. We denote $\mathcal{X}$ as the audio sample space and $\mathcal{Y}$ as the textual sample space. Since our goal is not comparing between different approaches but presenting the usage of PD estimation for cross-modal learning, we select only one approach {\em Probabilistic Classifier} as our objective for estimating PD. Note that we report the logarithm of PD, which is PMI in the results. One may refer to Supplementary for more details on training and datasets.

By definition, given an audio feature $x$ and a textual feature $y$, their point-wise dependency $r(x,y)$ measures their statistical dependency. For example, if $x_1$ and $y_1$ are the features for the same word, and $y_2$ is the feature for another word, then $r(x_1, y_1) > r(x_1, y_2)$ (in most cases). As a consequence, we can train PD estimators using the training split, and computing PD values for cross-modal retrieval on the test split.

\paragraph{Results \& Discussions} In Table~\ref{tbl:CM}, we report the results on $1:5$ matching\footnote{
One trial contains an audio feature, its corresponding textual feature, and $4$ randomly sampled textual features.} from audio to textual features. 
First, we obtain 96.24\% top-1 retrieval accuracy using PD estimation (with {\em Probabilistic Classifier} approach). Another approach such as {\em Density-Ratio Fitting} obtains 92.26\% top-1 retrieval accuracy. Then, we study the success and failure retrieval cases. The success examples show the highest statistical dependency (i.e., the highest PMI) between the audio and textual features of the same word. The failure examples, on the contrary, (all of them) have the second-highest PMI between the audio and textual features of the same word. Last, we observe that only the correctly retrieved cross-modal features have positive PMI values, which suggest two features are statistically dependent. As a summary, PD acts as a statistical dependency measurement, and we show its estimation can be generalized from training to test split for cross-modal retrieval.

\section{Conclusion}

In contrast to mutual information, which is an aggregate statistic of the dependency between two random variables, this paper contributes to present methods for estimating instance-level dependency. To overcome the curse of dimensionality in classical kernel-based approaches, we leverage the power of rich and flexible neural networks to model high-dimensional data. In particular, we first show that point-wise dependency is a natural product from optimizing mutual information variational bounds. Then, we further develop two point-wise dependency estimation approaches: Probabilistic Classifier and Density-Ratio Fitting that are free of optimizing mutual information variational bounds. A diversified set of experiments manifest the advantages of using our approaches. We believe this work sheds light on the advantages of estimating instance-level dependency between high-dimensional data, making a step forward towards improving unsupervised or cross-modal representation learning.

\section*{Broader Impact}
This paper presents methods for estimating point-wise dependency between high-dimensional data using neural networks. This work may benefit the applications that require understanding instance-level dependency. Take adversarial samples detection as an example: we can perform point-wise dependency estimation between data and label, and the ones with low point-wise dependency can be regarded as adversarial samples. We should also be aware of the malicious usage for our framework. For instance, people with bad intentions can use our framework to detect samples that have a high point-wise dependency with their of-interest private attributes. Then, these detected samples may be used for malicious purposes. 

\section*{Acknowledgement}
This work was supported in part by the Facebook Fellowship, DARPA grants FA875018C0150 
HR00111990016, NSF IIS1763562, NSF Awards \#1750439 \#1722822, National Institutes of Health, and Apple.
We would also like to acknowledge NVIDIA’s GPU support.

{
\small
\bibliography{density_ratio}
\bibliographystyle{plainnat}
}

\newpage

\section{Optimization Objectives for Point-wise Dependency Neural Estimation}

In this section, we shall show detailed derivations for the point-wise dependency estimation methods. Four approaches are discussed: {\em Variational Bounds of Mutual Information}, {\em Density Matching}, {\em Probabilistic Classifier}, and {\em Density-Ratio Fitting}. For convenience, we define $\Omega = \mathcal{X}\times \mathcal{Y}$. We have $P_{X,Y}$ and $P_X P_Y$ (can also be written as $P_X \otimes P_Y$) be the probability measures over $\sigma-$algebras over $\Omega$ with their probability densities being the Radon-Nikodym derivatives (i.e., $p(x,y) = dP_{X,Y}/d\mu$ and $p(x)p(y) = dP_X P_Y/d\mu$ with $\mu$ being the Lebesgue measure).

\subsection{Method I: Variational Bounds of Mutual Information}

Recent advances~\cite{belghazi2018mine,poole2019variational} propose to estimate mutual information (MI) using neural network either from variational MI lower bounds (e.g., $I_{\rm NWJ}$~\cite{belghazi2018mine} and $I_{\rm DV}$~\cite{belghazi2018mine}) or a variational form of MI (e.g., $I_{\rm JS}$~\cite{poole2019variational}). These estimators have the logarithm of point-wise dependency (PMI) as the intermediate product, which we will show in the following. We denote $\mathcal{M}$ be any class of functions $m:\Omega \rightarrow \mathbb{R}$.

\begin{prop}[$I_{\rm NWJ}$ and its neural estimation, restating Nguyen-Wainwright-Jordan bound~\cite{nguyen2010estimating,belghazi2018mine}] 
\begin{equation*}
I_{\rm NWJ}:=\underset{m \in \mathcal{M}}{\rm sup}\,\mathbb{E}_{P_{X,Y}}[m(x,y)] - e^{-1}\mathbb{E}_{P_{X}P_{Y}}[e^{m(x,y)}] = \underset{\theta \in \Theta}{\rm sup}\,\mathbb{E}_{P_{X,Y}}[\hat{f}_\theta (x,y)] - e^{-1}\mathbb{E}_{P_{X}P_{Y}}[e^{\hat{f}_\theta (x,y)}]
\end{equation*}
has the optimal function $m^*(x,y) = 1+{\rm log}\,\frac{p(x,y)}{p(x)p(y)}$. And when $\Theta$ is large enough, the optimal $\hat{f}^*_\theta(x,y) =  1+{\rm log}\,\frac{p(x,y)}{p(x)p(y)}$.
\label{prop:nwj}
\end{prop}
\begin{proof} The second-order functional derivative of the objective is $-e^{-1}\cdot e^{m(x,y)}\cdot dP_{X}P_{Y}$, which is always negative. The negative second-order functional derivative implies the objective has a supreme value. 
Then, take the first-order functional derivative $\frac{\partial I_{\rm NWJ}}{\partial m}$ and set it to zero:
\begin{equation*}
d P_{X,Y} - e^{-1} \cdot e^{m(x,y)}\cdot d P_{X}P_{Y} = 0.
\end{equation*}
We then get optimal $m^*(x,y) = 1+{\rm log}\,\frac{dP_{X,Y}}{dP_{X}P_{Y}} = 1+{\rm log}\,\frac{p(x,y)}{p(x)p(y)}$.
When $\Theta$ is large enough, by universal approximation theorem of neural networks~\cite{hornik1989multilayer}, the approximation in Proposition~\ref{prop:nwj} is tight, which means $\hat{f}^*_\theta(x,y) = m^*(x,y) = 1+{\rm log}\,\frac{p(x,y)}{p(x)p(y)}$. 
\end{proof}

\begin{prop}[$I_{\rm DV}$ and its neural estimation, restating Donsker-Varadhan bound~\cite{donsker1983asymptotic,belghazi2018mine}] 
\begin{equation*}
I_{\rm DV}:=\underset{m \in \mathcal{M}}{\rm sup}\,\mathbb{E}_{P_{X,Y}}[m(x,y)] - {\rm log}\,\Big(\mathbb{E}_{P_{X}P_{Y}}[e^{m(x,y)}]\Big)] = \underset{\theta \in \Theta}{\rm sup}\,\mathbb{E}_{P_{X,Y}}[\hat{f}_\theta(x,y)] - {\rm log}\,\Big(\mathbb{E}_{P_{X}P_{Y}}[e^{\hat{f}_\theta(x,y)}]\Big)]
\end{equation*}
has optimal functions $m^*(x,y) = {\rm log}\,\frac{p(x,y)}{p(x)p(y)} + {\rm Const.}$. And when $\Theta$ is large enough, the optimal $\hat{f}^*_\theta(x,y) =  {\rm log}\,\frac{p(x,y)}{p(x)p(y)} + {\rm Const.}$.
\label{prop:dv}
\end{prop}
\begin{proof} Let ${\bf 1}_{\cdot }$ be an indicator function, and the second-order functional derivative of the objective is
\begin{equation*}
-\frac{e^{m(x,y)}\cdot \mathbb{E}_{(x',y')\sim P_XP_Y}\Bigg[e^{m(x',y')}\cdot {\bf 1}_{(x',y')\neq (x,y)}\Bigg]}{\Big(\mathbb{E}_{P_XP_Y}[e^{m(x,y)}]\Big)^2}\cdot dP_XP_Y,
\end{equation*}
which is always negative. The negative second-order functional derivative implies the objective has a supreme value. Then, take the first-order functional derivative $\frac{\partial I_{\rm DV}}{\partial m}$ and set it to zero:
\begin{equation*}
d P_{X,Y} - \frac{e^{m(x,y)}}{\mathbb{E}_{P_XP_Y}[e^{m(x,y)}]} \cdot d P_{X}P_{Y} = 0.
\end{equation*}
We then have $m^*(x,y)$ take the forms $m^*(x,y) = {\rm log}\,\frac{dP_{X,Y}}{dP_{X}P_{Y}} + {\rm Const.} = {\rm log}\,\frac{p(x,y)}{p(x)p(y)} + {\rm Const.}$.
When $\Theta$ is large enough, by universal approximation theorem of neural networks~\cite{hornik1989multilayer}, the approximation in Proposition~\ref{prop:dv} is tight, which means $\hat{f}^*_\theta(x,y) = m^*(x,y) = {\rm log}\,\frac{p(x,y)}{p(x)p(y)} + {\rm Const.}$. 
\end{proof}

\begin{prop}[$I_{\rm JS}$ and its neural estimation, restating Jensen-Shannon bound with f-GAN objective~\cite{poole2019variational}]
\begin{equation*}
\begin{split}
    I_{\rm JS}:=& \underset{m \in \mathcal{M}}{\rm sup}
\,\mathbb{E}_{P_{X,Y}}\Big[-{\rm softplus} \Big(-{m}(x,y)\Big)\Big] - \mathbb{E}_{P_{X} P_{Y}}\Big[{\rm softplus} \Big({m}(x,y)\Big)\Big]  
    \\
    & = \underset{\theta \in \Theta}{\rm sup}
\,\mathbb{E}_{P_{X,Y}}\Big[-{\rm softplus} \Big(-\hat{f}_\theta(x,y)\Big)\Big] - \mathbb{E}_{P_{X} P_{Y}}\Big[{\rm softplus} \Big(\hat{f}_\theta(x,y)\Big)\Big]
\end{split}
\end{equation*}
with ${\rm softplus}$ function being ${\rm softplus}\,(x) = {\rm log}\,\Big(1+{\rm exp}\,(x)\Big)$ and the optimal solution $m^*(x,y) = {\rm log}\,\frac{p(x,y)}{p(x)p(y)}$. And when $\Theta$ is large enough, the optimal $\hat{f}^*_\theta(x,y) = m^*(x,y) =  {\rm log}\,\frac{p(x,y)}{p(x)p(y)}$.
\label{prop:js}
\end{prop}
\begin{proof} The second-order functional derivative of the objective is
\begin{equation*}
-\frac{1}{\Big(1+e^{m(x,y)}\Big)^2}\cdot e^{m(x,y)} \cdot dP_{X,Y} - \frac{1}{\Big(1+e^{-m(x,y)}\Big)^2}\cdot e^{-m(x,y)} \cdot dP_XP_Y,
\end{equation*}
which is always negative. The negative second-order functional derivative implies the objective has a supreme value. 
Then, take the first-order functional derivative $\frac{\partial I_{\rm JS}}{\partial m}$ and set it to zero:
\begin{equation*}
\frac{1}{1+e^{-m(x,y)}}\cdot e^{-m(x,y)} \cdot dP_{X,Y} - \frac{1}{1+e^{m(x,y)}}\cdot e^{m(x,y)} \cdot dP_XP_Y =0.
\end{equation*}
We then get $m^*(x,y) = {\rm log}\,\frac{dP_{X,Y}}{dP_{X}P_{Y}} = {\rm log}\,\frac{p(x,y)}{p(x)p(y)}$. 
When $\Theta$ is large enough, by universal approximation theorem of neural networks~\cite{hornik1989multilayer}, the approximation in Proposition~\ref{prop:js} is tight, which means $\hat{f}^*_\theta(x,y) = m^*(x,y) = {\rm log}\,\frac{p(x,y)}{p(x)p(y)}$. 
\end{proof}

We see that either $I_{\rm NWJ}$ (Proposition~\ref{prop:nwj}) or $I_{\rm JS}$ (Proposition~\ref{prop:js}) gives us the optimal PMI estimation, while $I_{\rm DV}$ (Proposition~\ref{prop:dv}) is less preferable since its optimal solution includes an arbitrary constant. In practice, we prefer $I_{\rm JS}$ over $I_{\rm NWJ}$/$I_{\rm DV}$ due to its better training stability~\cite{poole2019variational}.

\subsection{Method II: Density Matching}

This method considers to match the true joint density $p(x,y)$ and the estimated joint density via KL-divergence. We let the estimated joint probability be $P_{m,X,Y}$ with its joint density being $e^{m(x,y)}p(x)p(y)$, where $e^{m(x,y)}$ acts to ensure the estimated joint density is a valid probability density function. Hence, we let $m\in \mathcal{M}''$ with $\mathcal{M}''$ being 1) any class of functions $m: \Omega \rightarrow \mathbb{R}$; and 2) $\int e^{m(x,y)}\,dP_XP_Y = \mathbb{E}_{P_XP_Y}[e^{m(x,y)}] = 1$. 
\begin{prop}[KL Loss in Density Matching and its neural estimation]
\begin{equation*}
\begin{split}
L_{\rm KL_{DM}} := & \underset{m\in \mathcal{M}''}{\rm sup}\mathbb{E}_{P_{X,Y}}[m(x,y)]\\
= &~ \underset{\theta \in \Theta }{\rm sup}\,\, \mathbb{E}_{P_{X,Y}}[\hat{f}_\theta (x,y)]\,\, {\rm s.t.}\,\,\mathbb{E}_{P_X  P_Y}[e^{\hat{f}_\theta(x,y)}]=1
\end{split}
\end{equation*}
with the optimal $m^*(x,y)= {\rm log}\,\frac{p(x,y)}{p(x)p(y)}$. And when $\Theta$ is large enough, the optimal $\hat{f}^*_\theta(x,y) = {\rm log}\, \frac{p(x,y)}{p(x)p(y)}$. 
\label{prop:KL_DM}
\end{prop}
\begin{proof}
First, we compute the KL-divergence:
\begin{equation*}
\begin{split}
L_{\rm KL_{DM}} =\, &\underset{m\in \mathcal{M}''}{\rm inf}\,\,\kldiv{P_{X,Y}}{\hat{P}_{X,Y}} =  \underset{m\in \mathcal{M}''}{\rm inf}\,\, H(P_{X,Y}) - \mathbb{E}_{P_{X,Y}}\Big[{\rm log}\,e^{m(x,y)}p(x)p(y) \Big] \\
=\,& \underset{m\in \mathcal{M}''}{\rm inf}\,\, H(P_{X,Y}) - \mathbb{E}_{P_{X,Y}}\Big[{\rm log}\,p(x)p(y) \Big] - \mathbb{E}_{P_{X,Y}}\Big[m(x,y) \Big]
 \\
 =\,&\underset{m\in \mathcal{M}''}{\rm inf}\,\,I(X;Y) - 
 \mathbb{E}_{P_{X,Y}}\Big[m(x,y) \Big] = {\rm Const.} + \underset{m\in \mathcal{M}''}{\rm sup}\,\,
 \mathbb{E}_{P_{X,Y}}\Big[m(x,y) \Big] 
 \\
 \Leftrightarrow\,& \underset{m\in \mathcal{M}}{\rm sup}\,\mathbb{E}_{P_{X,Y}}[m(x,y)]\,\,{\rm s.t.}\,\,\mathbb{E}_{P_X  P_Y}[e^{m(x,y)}]=1.
\end{split}
\end{equation*}

Consider the following Lagrangian:
\begin{equation*}
\begin{split}
h(m, \lambda_1, \lambda_2) := \mathbb{E}_{P_{X,Y}}[m] - \lambda (\mathbb{E}_{P_X  P_Y}[e^m]-1),
\end{split}
\end{equation*}
where $\lambda \in \mathbb{R}$. Taking the functional derivative and setting it to be zero, we see
\begin{equation*}
\begin{split}
    dP_{X,Y} - \lambda \cdot e^m\cdot dP_{X}dP_{Y} = 0.
\end{split}
\end{equation*}
To satisfy the constraint, we obtain
\begin{equation*}
\begin{split}
    \mathbb{E}_{P_X  P_Y}[e^m] = 1 \iff {E}_{P_X  P_Y}[\frac{1}{\lambda}\frac{dP_{X,Y}}{dP_XP_Y}] = \frac{1}{\lambda} {E}_{P_X  P_Y}[\frac{dP_{X,Y}}{dP_XP_Y}] = \frac{1}{\lambda} = 1 \iff \lambda =1.
\end{split}
\end{equation*}
Plugging-in $\lambda = 1$, the optimal $m^*(x,y) = {\rm log}\, \frac{dP_{XY}}{dP_XP_Y} = {\rm log}\, \frac{p(x,y)}{p(x)p(y)}$. When $\Theta$ is large enough, by universal approximation theorem of neural networks~\cite{hornik1989multilayer}, the approximation in Proposition~\ref{prop:KL_DM} is tight, which means $\hat{f}^*_\theta(x,y) = m^*(x,y) = {\rm log}\,\frac{p(x,y)}{p(x)p(y)}$.
\end{proof}

The objective function in Proposition~\ref{prop:KL_DM} is a constrained optimization problem, and we present two relaxed optimization objectives. The first one is Lagrange relaxation:
\begin{equation*}
\small
\begin{split}
 \underset{\theta \in \Theta }{\rm sup}\,\, \mathbb{E}_{P_{X,Y}}[\hat{f}_\theta (x,y)] - \lambda \Big( \mathbb{E}_{P_X P_Y}[e^{\hat{f}_\theta(x,y)}] - 1 \Big) 
\end{split}
\end{equation*}
with the optimal Lagrange coefficient $\lambda = 1$ (see proof for Proposition~\ref{prop:KL_DM}). 

The second one is log barrier method:
\begin{equation*}
\small
\begin{split}
\underset{\theta \in \Theta}{\rm sup}\,\, \mathbb{E}_{P_{X,Y}}[\hat{f}_\theta (x,y)] - \eta\Big({\rm log}\, \mathbb{E}_{P_X P_Y}[
e^{\hat{f}_\theta(x,y)}]\Big)^2,
\end{split}
\end{equation*}
where $\eta > 0$ is a hyper-parameter controlling the regularization term.

\subsection{Method III: Probabilistic Classifier}
This approach casts the PD estimation as the problem of estimating the `class'-posterior probability. We use a Bernoulli random variable $C$ to classify the samples drawn from the joint density ($C=1$ for $(x,y)\sim P_{X,Y}$) and the samples drawn from product of the marginal densities ($C=0$ for $(x,y)\sim P_{X}P_{Y}$). In order to present our derivation, we define $H(\cdot)$ as the entropy and $H(\cdot , \cdot)$ as the cross entropy. Slightly abusing notation, in this subsection, we define $\Omega' = \mathcal{X}\times \mathcal{Y}\times \{0,1\}$ and $\mathcal{M}'$ is 1) any class of functions $m:\Omega ' \rightarrow (0,1)$; and 2) $m(x,y,0) + m(x,y,1)=1$ for any $x$ and $y$. Note that since $m(x,y,c)$ is always positive and $m(x,y,0)+m(x,y,1)=1$ for any $x,y$, $m(x,y,c)$ is a proper probability mass function with respect to $C$ given any $x,y$. Consider the binary cross entropy loss:
\begin{prop}[Binary Cross Entropy Loss in Probabilistic Classifier Method and its neural estimation]
\begin{equation*}
\begin{split}
L_{\rm BCE_{PC}} := & \underset{m \in \mathcal{M}'}{\rm sup}\,  \mathbb{E}_{P_{X,Y}}[{\rm log}\,m(x,y, C=1)]  + \mathbb{E}_{P_{X}P_{Y}}[{\rm log}\,\Big(1-m(x,y, C=1)\Big)] 
\\
= &\underset{\theta \in \Theta}{\rm sup}\,  \mathbb{E}_{P_{X,Y}}[{\rm log}\,\hat{p}_\theta (C=1|(x,y))]  + \mathbb{E}_{P_{X}P_{Y}}[{\rm log}\,\Big(1-\hat{p}_\theta (C=1|(x,y))\Big)] 
\end{split}
\end{equation*}
\label{prop:pc}
\end{prop}
with the optimal $m^*(x,y, c)=p(c|(x,y))$. And when $\Theta$ is large enough, the optimal $\hat{p}^*_\theta(c|(x,y)) = p(c|(x,y))$. 
\begin{proof}
We see 
\begin{equation*}
\begin{split}
L_{\rm BCE_{PC}} = &
\underset{m \in \mathcal{M}'}{\rm inf}\,  \mathbb{E}_{P_{XY}}\Big[ H\Big(P(C|(x,y)), m(x,y, C)\Big) \Big] + \mathbb{E}_{P_{X}P_{Y}}\Big[ H\Big(P(C|(x,y)), m(x,y, C))\Big) \Big]
\\
= & \underset{m \in \mathcal{M}'}{\rm inf}\,  \mathbb{E}_{P_{XY}}\Big[ H\Big(P(C|(x,y))\Big) +
\kldiv{P(C|(x,y))}{m(x,y, C)} \Big] 
\\
& \quad \quad + \mathbb{E}_{P_{X}P_{Y}}\Big[ H\Big(P(C|(x,y))\Big) + \kldiv{P(C|(x,y))}{m(x,y, C)} \Big]
\\
= & \,{\rm Const.} + \underset{m \in \mathcal{M}'}{\rm inf}\,   \mathbb{E}_{P_{XY}}\Big[ \kldiv{P(C|(x,y))}{m(x,y, C)} \Big] 
\\
& \quad \quad \quad \quad \quad \quad  + \mathbb{E}_{P_{X}P_{Y}}\Big[ \kldiv{P(C|(x,y))}{m(x,y, C)} \Big] 
\\
= & \,{\rm Const.} + \underset{m \in \mathcal{M}'}{\rm inf}\,  \mathbb{E}_{P_{XY}}\Big[ \mathbb{E}_{P(C|(x,y))}[-{\rm log}\,m(x,y,c)]\Big]
\\
& \quad \quad \quad \quad \quad \quad  + \mathbb{E}_{P_{X}P_{Y}}\Big[ \mathbb{E}_{P(C|(x,y))}[-{\rm log}\,m(x,y,c)]\Big]  
\\
= & \,{\rm Const.} + \underset{m \in \mathcal{M}'}{\rm inf}\,  \mathbb{E}_{P_{XY}}[-{\rm log}\,m(x,y,C=1)]+\mathbb{E}_{P_{X}P_{Y}}[-{\rm log}\,m(x,y,C=0)]
\\
\Leftrightarrow & \underset{m \in \mathcal{M}'}{\rm sup}\,  \mathbb{E}_{P_{X,Y}}[{\rm log}\,m(x,y,C=1)]  + \mathbb{E}_{P_{X}P_{Y}}[{\rm log}\,\Big(1-m(x,y,C=1)\Big)].
\end{split}
\end{equation*}
The optimal $m^*$ happens when $\kldiv{P(C|(x,y))}{m^*(x,y, C)}=0$ for any $(x,y)$, which implies $m^*(x,y,c)=p(c|(x,y))$. When $\Theta$ is large enough, by universal approximation theorem of neural networks~\cite{hornik1989multilayer}, the approximation in Proposition~\ref{prop:pc} is tight, which means $\hat{p}^*_\theta(c|(x,y)) = m^*(x,y,c) = p(c|(x,y))$.
\end{proof}

The obtained estimated class-posterior classifier can be used for approximating point-wise dependency (PD):
\begin{equation*}
\hat{r}_\theta (x,y)  = \frac{n_{P_XP_Y}}{n_{P_{X,Y}}} \frac{\hat{p}_\theta (C=1|(x,y))}{\hat{p}_\theta(C=0|(x,y))}\,\,{\rm with}\,\,(x,y)\sim P_{X,Y} \,\,{\rm or}\,\,(x,y)\sim P_{X}P_{Y}.
\end{equation*}

\subsection{Method IV: Density-Ratio Fitting}

Let $\mathcal{M}$ be any class of functions $m:\Omega \rightarrow \mathbb{R}$.
This approach considers to minimize the expected (in $\mathbb{E}_{P_XP_Y}[\cdot]$) least-square difference between the true PD $r(x,y)$ and the estimated PD $m(x,y)$: 
\begin{prop}[Least-Square Loss in Density-Ratio Fitting and its neural estimation]
\begin{equation*}
L_{\rm LS_{D-RF}}:=  \underset{m\in \mathcal{M}}{\rm sup}\,\mathbb{E}_{P_{X,Y}}[m(x,y)] - \frac{1}{2}\mathbb{E}_{P_XP_Y}[m^2(x,y)] = \underset{\theta \in \Theta}{\rm sup}\,\mathbb{E}_{P_{X,Y}}[\hat{r}_\theta(x,y)] - \frac{1}{2}\mathbb{E}_{P_XP_Y}[\hat{r}^2_\theta(x,y)]
\end{equation*}
with the optimal $m^*(x,y)=\frac{p(x,y)}{p(x)p(y)}$. And when $\Theta$ is larger enough, the optimal $\hat{r}^*_\theta(x,y) = \frac{p(x,y)}{p(x)p(y)}$.
\label{prop:LS_D-RF}
\end{prop}
\begin{proof}
\begin{equation*}
\begin{split}
L_{\rm LS_{D-RF}} = & \underset{m \in \mathcal{M}}{\rm inf}\,\,\mathbb{E}_{P_X  P_Y} [\big(r(x,y) - m(x,y)\big)^2]
\\
= & \underset{m \in \mathcal{M}}{\rm inf}\,\,\mathbb{E}_{P_X  P_Y} [r^2(x,y)] -2\mathbb{E}_{P_X  P_Y} [r(x,y)m(x,y)] + \mathbb{E}_{P_X  P_Y} [m^2(x,y)]
\\
= & \,{\rm Const.} + \underset{m \in \mathcal{M}}{\rm inf}\,\,-2\mathbb{E}_{P_X  P_Y} [r(x,y)m(x,y)] + \mathbb{E}_{P_X  P_Y} [m^2(x,y)]
\\
= & \,{\rm Const.} + \underset{m \in \mathcal{M}}{\rm inf}\,\,-2\mathbb{E}_{P_{XY}} [m(x,y)] + \mathbb{E}_{P_X  P_Y} [m^2(x,y)]
\\
\Leftrightarrow \, & \underset{m \in \mathcal{M}}{\rm sup}\,\,\mathbb{E}_{P_{XY}} [m(x,y)] -\frac{1}{2} \mathbb{E}_{P_X  P_Y} [m^2(x,y)].
\end{split}
\end{equation*}
Take the first-order functional derivative and set it to zero:
\begin{equation*}
dP_{XY} - m(x,y)\cdot dP_X P_Y = 0. 
\end{equation*}
We then get $m^*(x,y) = \frac{dP_{X,Y}}{dP_{X}P_{Y}}=\frac{p(x,y)}{p(x)p(y)}$.
When $\Theta$ is large enough, by universal approximation theorem of neural networks~\cite{hornik1989multilayer}, the approximation in Proposition~\ref{prop:LS_D-RF} is tight, which means $\hat{r}^*_\theta(x,y) = m^*(x,y) = \frac{p(x,y)}{p(x)p(y)}$. 
\end{proof}

\section{More on Mutual Information Neural Estimation}

In this section, we present more analysis on estimating mutual information (MI) using neural networks. Before going into more details, we would like to 1) show $I_{\rm NWJ}$ and $I_{\rm DV}$ are MI lower bounds; and 2) present $I_{\rm CPC}$~\cite{oord2018representation} objective.

\begin{lemma}[$I_{\rm NWJ}$ as a MI lower bound] 
\begin{equation*}
    \forall \theta \in \Theta , \quad I(X;Y) \geq
    \mathbb{E}_{P_{X,Y}}[\hat{f}_\theta (x,y)] - e^{-1}\mathbb{E}_{P_{X}P_{Y}}[e^{\hat{f}_\theta (x,y)}].
\end{equation*}
Therefore, 
\begin{equation*}
    I(X;Y) \geq I_{\rm NWJ}:= \underset{\theta \in \Theta}{\rm sup}\,
    \mathbb{E}_{P_{X,Y}}[\hat{f}_\theta (x,y)] - e^{-1}\mathbb{E}_{P_{X}P_{Y}}[e^{\hat{f}_\theta (x,y)}].
\end{equation*}
\label{lemma:nwj}
\end{lemma}
\begin{proof}
In Proposition~\ref{prop:nwj}, we show the supreme value of $\mathbb{E}_{P_{X,Y}}[\hat{f}_\theta (x,y)] - e^{-1}\mathbb{E}_{P_{X}P_{Y}}[e^{\hat{f}_\theta (x,y)}]$ happens when $\hat{f}_\theta^* (x,y) = 1+{\rm log}\,\frac{p(x,y)}{p(x)p(y)}$. Plugging-in $\hat{f}_\theta^* (x,y)$, we get
\begin{align*}
&\mathbb{E}_{P_{X,Y}}[\hat{f}^*_\theta (x,y)] - e^{-1}\mathbb{E}_{P_{X}P_{Y}}[e^{\hat{f}^*_\theta (x,y)}] = \mathbb{E}_{P_{X,Y}}[1+{\rm log}\,\frac{p(x,y)}{p(x)p(y)}] - e^{-1}\mathbb{E}_{P_{X}P_{Y}}[e^1\cdot \frac{p(x,y)}{p(x)p(y)}]\\
=& 1 + \mathbb{E}_{P_{X,Y}}[{\rm log}\,\frac{p(x,y)}{p(x)p(y)}] - e^{-1}\cdot e^{1} \cdot \mathbb{E}_{P_{X}P_{Y}}[\frac{p(x,y)}{p(x)p(y)}] = 1 + I(X;Y) - 1 = I(X;Y).\qedhere
\end{align*}    
\end{proof}

\begin{lemma}[$I_{\rm DV}$ as a MI lower bound] 
\begin{equation*}
    \forall \theta \in \Theta , \quad I(X;Y) \geq
    \mathbb{E}_{P_{X,Y}}[\hat{f}_\theta (x,y)] - {\rm log}\,\Big( \mathbb{E}_{P_{X}P_{Y}}[e^{\hat{f}_\theta (x,y)}]\Big).
\end{equation*}
Therefore, 
\begin{equation*}
    I(X;Y) \geq I_{\rm DV}:= \underset{\theta \in \Theta}{\rm sup}\,
    \mathbb{E}_{P_{X,Y}}[\hat{f}_\theta (x,y)] - - {\rm log}\, \Big( \mathbb{E}_{P_{X}P_{Y}}[e^{\hat{f}_\theta (x,y)}]\Big).
\end{equation*}
\label{lemma:dv}
\end{lemma}
\begin{proof}
In Proposition~\ref{prop:dv}, we show the supreme value of $\mathbb{E}_{P_{X,Y}}[\hat{f}_\theta (x,y)] - {\rm log}\, \Big(\mathbb{E}_{P_{X}P_{Y}}[e^{\hat{f}_\theta (x,y)}]\Big)$ happens when $\hat{f}_\theta^* (x,y) = {\rm Const.}+{\rm log}\,\frac{p(x,y)}{p(x)p(y)}$. Plugging-in $\hat{f}_\theta^* (x,y)$, we get
\begin{equation*}
\begin{split}
&\mathbb{E}_{P_{X,Y}}[\hat{f}^*_\theta (x,y)] - {\rm log}\, \Big(\mathbb{E}_{P_{X}P_{Y}}[e^{\hat{f}^*_\theta (x,y)}]\Big) \\
=& \mathbb{E}_{P_{X,Y}}[{\rm Const.}+{\rm log}\,\frac{p(x,y)}{p(x)p(y)}] - {\rm log}\, \Big(\mathbb{E}_{P_{X}P_{Y}}[e^{{\rm Const.}+{\rm log}\,\frac{p(x,y)}{p(x)p(y)}}]\Big)\\
=& {\rm Const.} + \mathbb{E}_{P_{X,Y}}[{\rm log}\,\frac{p(x,y)}{p(x)p(y)}] - {\rm Const.} \cdot \mathbb{E}_{P_{X}P_{Y}}[\frac{p(x,y)}{p(x)p(y)}] = I(X;Y).
\end{split}    
\end{equation*}
\qedhere
\end{proof}

\begin{prop}[$I_{\rm CPC}$, restating Contrastive Predictive Coding~\cite{oord2018representation}] With $\hat{c}_\theta (x,y)$ representing a real-valued measureable function on $\mathcal{X}\times \mathcal{Y}$ which is parametrized by a neural network $\theta$,
\begin{equation*}
\begin{split}
L_{\rm CPC} :=  \underset{\theta \in \Theta}{\rm sup}\,  \mathbb{E}_{(x_1,y_1)\sim P_{X,Y}, \cdots (x_n,y_n)\sim P_{X,Y}}[\frac{1}{n}\sum_{i=1}^{n}{\rm log}\,\frac{e^{\hat{c}_\theta(x_i, y_i)}}{\frac{1}{n}\sum_{j=1}^{n}e^{\hat{c}_\theta(x_i, y_j)}}]
\end{split}
\end{equation*} 
with an upper bound value ${\rm log}\,n$.
\label{prop:cpc}
\end{prop}
\begin{proof}
\begin{equation*}
\begin{split}
L_{\rm CPC} & =  \underset{\theta \in \Theta}{\rm sup}\,  \mathbb{E}_{(x_1,y_1)\sim P_{X,Y}, \cdots (x_n,y_n)\sim P_{X,Y}}[\frac{1}{n}\sum_{i=1}^{n}{\rm log}\,\frac{e^{\hat{c}_\theta(x_i, y_i)}}{\frac{1}{n}\sum_{j=1}^{n}e^{\hat{c}_\theta(x_i, y_j)}}]\\
& =  \underset{\theta \in \Theta}{\rm sup}\,  \mathbb{E}_{(x_1,y_1)\sim P_{X,Y}, \cdots (x_n,y_n)\sim P_{X,Y}}[\frac{1}{n}\sum_{i=1}^{n}{\rm log}\,\frac{e^{\hat{c}_\theta(x_i, y_i)}}{\sum_{j=1}^{n}e^{\hat{c}_\theta(x_i, y_j)}}] + {\rm log}\, n \\
& \leq \underset{\theta \in \Theta}{\rm sup}\,  \mathbb{E}_{(x_1,y_1)\sim P_{X,Y}, \cdots (x_n,y_n)\sim P_{X,Y}}[\frac{1}{n}\sum_{i=1}^{n}{\rm log}\,\frac{e^{\hat{c}_\theta(x_i, y_i)}}{e^{\hat{c}_\theta(x_i, y_i)}}] + {\rm log}\, n \\
& = \underset{\theta \in \Theta}{\rm sup}\,  \mathbb{E}_{(x_1,y_1)\sim P_{X,Y}, \cdots (x_n,y_n)\sim P_{X,Y}}[\frac{1}{n}\sum_{i=1}^{n}{\rm log}\,1] + {\rm log}\, n  \\
& = {\rm log}\, n .
\end{split}
\end{equation*} 
\end{proof}

\begin{lemma}[$I_{\rm CPC}$ as a MI lower bound]
\begin{equation*}
    \forall \theta \in \Theta , \quad I(X;Y) \geq
    \mathbb{E}_{(x_1,y_1)\sim P_{X,Y}, \cdots (x_n,y_n)\sim P_{X,Y}}[\frac{1}{n}\sum_{i=1}^{n}{\rm log}\,\frac{e^{\hat{c}_\theta(x_i, y_i)}}{\frac{1}{n}\sum_{j=1}^{n}e^{\hat{c}_\theta(x_i, y_j)}}].
\end{equation*}
Therefore, 
\begin{equation*}
    I(X;Y) \geq I_{\rm CPC}:=\underset{\theta \in \Theta}{\rm sup}\,
    \mathbb{E}_{(x_1,y_1)\sim P_{X,Y}, \cdots (x_n,y_n)\sim P_{X,Y}}[\frac{1}{n}\sum_{i=1}^{n}{\rm log}\,\frac{e^{\hat{c}_\theta(x_i, y_i)}}{\frac{1}{n}\sum_{j=1}^{n}e^{\hat{c}_\theta(x_i, y_j)}}].
\end{equation*}
\label{lemma:cpc}
\end{lemma}
\begin{proof}
First, we use independent and identical random variables $X_1 , X_2 , \cdots ,X_n$ and $Y_1, Y_2, \cdots , Y_n$ to represent the copies of $X$ and $Y$, where $(x_i , y_i)\sim P_{X_i , Y_i}$. Replacing the random variables in Lemma~\ref{lemma:nwj}, we obtain 
\begin{equation*}
\forall \theta \in \Theta , \quad I(X_i;Y_{1:n}) \geq
    \mathbb{E}_{P_{X_{i},Y_{1:n}}}[\hat{f}_\theta (x_i,y_{1:k})] - e^{-1}\mathbb{E}_{P_{X_i}P_{Y_{1:n}}}[e^{\hat{f}_\theta (x_i,y_{1:k})}].    
\end{equation*}
Next, we define $\hat{f}_\theta (x_i,y_{1:k}) = 1 +{\rm log}\,\frac{e^{\hat{c}_\theta(x_i, y_i)}}{\frac{1}{n}\sum_{j=1}^{n}e^{\hat{c}_\theta(x_i, y_j)}}$ and get
\begin{equation*}
\begin{split}
\forall \theta \in \Theta , \quad I(X_i;Y_{1:n}) \geq 1 + 
    \mathbb{E}_{P_{X_{i},Y_{1:n}}}[{\rm log}\,\frac{e^{\hat{c}_\theta(x_i, y_i)}}{\frac{1}{n}\sum_{j=1}^{n}e^{\hat{c}_\theta(x_i, y_j)}}] - \mathbb{E}_{P_{X_i}P_{Y_{1:n}}}[\frac{e^{\hat{c}_\theta(x_i, y_i)}}{\frac{1}{n}\sum_{j=1}^{n}e^{\hat{c}_\theta(x_i, y_j)}}].
\end{split}
\end{equation*}
Since $Y_1 , Y_2 , \cdots , Y_n$ are independent and identical samples, $\mathbb{E}_{P_{X_i}P_{Y_{1:n}}}[\frac{e^{\hat{c}_\theta(x_i, y_i)}}{\frac{1}{n}\sum_{j=1}^{n}e^{\hat{c}_\theta(x_i, y_j)}}] = \mathbb{E}_{P_{X_i}P_{Y_{1:n}}}[\frac{e^{\hat{c}_\theta(x_i, y_{i'})}}{\frac{1}{n}\sum_{j=1}^{n}e^{\hat{c}_\theta(x_i, y_j)}}] \,\, \forall i' \in \{1, 2, \cdots , n\}$. Therefore, $\mathbb{E}_{P_{X_i}P_{Y_{1:n}}}[\frac{e^{\hat{c}_\theta(x_i, y_i)}}{\frac{1}{n}\sum_{j=1}^{n}e^{\hat{c}_\theta(x_i, y_j)}}] = \frac{1}{n}\sum_{i'=1}^n \mathbb{E}_{P_{X_i}P_{Y_{1:n}}}[\frac{e^{\hat{c}_\theta(x_i, y_{i'})}}{\frac{1}{n}\sum_{j=1}^{n}e^{\hat{c}_\theta(x_i, y_j)}}] = \mathbb{E}_{P_{X_i}P_{Y_{1:n}}}[\frac{\frac{1}{n}\sum_{i'=1}^n e^{\hat{c}_\theta(x_i, y_{i'})}}{\frac{1}{n}\sum_{j=1}^{n}e^{\hat{c}_\theta(x_i, y_j)}}] = 1$. Plugging-in this result, we have
\begin{equation*}
\begin{split}
\forall \theta \in \Theta , \quad I(X_i;Y_{1:n}) \geq 1 + 
    \mathbb{E}_{P_{X_{i},Y_{1:n}}}[{\rm log}\,\frac{e^{\hat{c}_\theta(x_i, y_i)}}{\frac{1}{n}\sum_{j=1}^{n}e^{\hat{c}_\theta(x_i, y_j)}}] - 1 =  \mathbb{E}_{P_{X_{i},Y_{1:n}}}[{\rm log}\,\frac{e^{\hat{c}_\theta(x_i, y_i)}}{\frac{1}{n}\sum_{j=1}^{n}e^{\hat{c}_\theta(x_i, y_j)}}].
\end{split}
\end{equation*}
Note that $Y_{i'}$ is independent to $X_i$ when $i' \neq i$, and therefore $I(X_i;Y_{1:n}) = I(X_i;Y_i) = I(X;Y)$. 

Bringing everything together, the original objective can be reformulated as
\begin{equation*}
\begin{split}
    & \mathbb{E}_{(x_1,y_1)\sim P_{X,Y}, \cdots (x_n,y_n)\sim P_{X,Y}}[\frac{1}{n}\sum_{i=1}^{n}{\rm log}\,\frac{e^{\hat{c}_\theta(x_i, y_i)}}{\frac{1}{n}\sum_{j=1}^{n}e^{\hat{c}_\theta(x_i, y_j)}}] \\
   = &  \mathbb{E}_{P_{X_{1:n},Y_{1:n}}}[\frac{1}{n}\sum_{i=1}^{n}{\rm log}\,\frac{e^{\hat{c}_\theta(x_i, y_i)}}{\frac{1}{n}\sum_{j=1}^{n}e^{\hat{c}_\theta(x_i, y_j)}}] = \frac{1}{n}\sum_{i=1}^{n} \mathbb{E}_{P_{X_{i},Y_{1:n}}}[{\rm log}\,\frac{e^{\hat{c}_\theta(x_i, y_i)}}{\frac{1}{n}\sum_{j=1}^{n}e^{\hat{c}_\theta(x_i, y_j)}}] \\
   \leq & \frac{1}{n}\sum_{i=1}^n I(X_i; Y_{1:n}) = \frac{1}{n}\sum_{i=1}^n I(X; Y) = I(X;Y).
\end{split}
\end{equation*}
\end{proof}

\subsection{Learning/ Inference in MI Neural Estimation and Baselines}

The MI neural estimation methods can be dissected into two procedures: {\em learning} and {\em inference}. The learning step learns the parameters when estimating 1) point-wise dependency (PD)/ logarithm of point-wise dependency (PMI); or 2) MI lower bound. The inference step considers the parameters from the learning step and infers value for 1) MI itself; or 2) a lower bound of MI. We summarize different approaches in Table 1 in the main text, and we discuss the baselines in this subsection. We present the comparisons between baselines and our methods in Table 1/ Figure 1 in the main text. 

\paragraph{CPC} Oord {\em et al.}~\cite{oord2018representation} presented {\bf C}ontrastive {\bf P}redictive {\bf C}oding ({\bf CPC}) as an unsupervised learning objective, which adopts $I_{\rm CPC}$ (see Proposition~\ref{prop:cpc}) in both learning and inference stages. From Proposition~\ref{prop:cpc} and Lemma~\ref{lemma:cpc}, we conclude 
\begin{equation*}
    I_{\rm CPC} \leq {\rm min}\,\Big({\rm log}\,n , I (X;Y) \Big).
\end{equation*}
Hence, the difference between $I_{\rm CPC}$ and $I(X;Y)$ is large when $n$ is small. This fact implies a large bias when using $I_{\rm CPC}$ to estimate MI. Nevertheless, empirical evidences~\cite{poole2019variational,song2019understanding} showed that $I_{\rm CPC}$ has low variance, which is also verified in our experiments.

\paragraph{NWJ} Belghazi {\em et al.}~\cite{belghazi2018mine} presented to use neural networks to estimate {\bf N}guyen-{\bf W}ainwright-{\bf J}ordan bound~\cite{nguyen2010estimating,belghazi2018mine} ({\bf NWJ}) bound of MI, which adopts $I_{\rm NWJ}$ (see Proposition~\ref{prop:nwj}) in both learning and inference stages. In Proposition~\ref{prop:nwj} and Lemma~\ref{lemma:nwj}, we show that when $\Theta$ is large enough, the supreme value of $I_{\rm NWJ}$ is $I(X;Y)$. Hence, we can expect a smaller bias when comparing $I_{\rm NWJ}$ to $I_{\rm CPC}$. Song {\em et al.}~\cite{song2019understanding} acknowledged the variance of an empirical $I_{\rm NWJ}$ estimation is $\Omega (e^{I(X;Y)})$, suggesting a large variance when the true MI is large. We verify these facts in our experiments.

\paragraph{DV (MINE)} Belghazi {\em et al.}~\cite{belghazi2018mine} presented to use neural networks to estimate {\bf D}onsker-{\bf V}aradhan bound~\cite{nguyen2010estimating,belghazi2018mine} ({\bf DV}) bound of MI, which adopts $I_{\rm DV}$ (see Proposition~\ref{prop:dv}) in both learning and inference stages. The author also refers this MI estimation procedure as {\bf M}utual {\bf I}nformation {\bf N}eural {\bf E}stimation ({\bf MINE}). In Proposition~\ref{prop:dv} and Lemma~\ref{lemma:dv}, we show that when $\Theta$ is large enough, the supreme value of $I_{\rm DV}$ is $I(X;Y)$. Hence, we can expect a smaller bias when comparing $I_{\rm DV}$ to $I_{\rm CPC}$. Song {\em et al.}~\cite{song2019understanding} acknowledged the limiting variance of an empirical $I_{\rm DV}$ estimation is $\Omega (e^{I(X;Y)})$, which implies the variance is large when the true MI is large. We verify these facts in our experiments.

\paragraph{JS} Unlike {\bf CPC}, {\bf NWJ}, and {\bf DV}, Poole {\em et al.}~\cite{poole2019variational} presented to adopt different objectives in learning and inference stages for MI estimation. Precisely, the author uses Jensen-Shannon F-GAN~\cite{nowozin2016f} objective (see Proposition~\ref{prop:js}) to estimate PMI and then plugs in the PMI into $I_{\rm NWJ}$ (see Proposition~\ref{prop:nwj}) for the inference. The author refers this MI estimation method as {\bf JS} since it considers {\bf J}ensen-{\bf S}hannon divergence during learning. Unfortunately, this estimation method still considers $I_{\rm NWJ}$ as its inference objective, and therefore the variance is still $\Omega (e^{I(X;Y)})$. Empirical results are shown in our experiments.

\paragraph{SMILE} To overcome the large variance issue in {\bf NWJ}, {\bf DV}, and {\bf JS}, Song {\em et al.}~\cite{song2019understanding} presented to use $I_{\rm JS}$ (see Proposition~\ref{prop:js}) for estimating PMI and then plug in the PMI to a modified $I_{\rm DV}$ (see Proposition~\ref{prop:dv}). Specifically, the author clipped the value of $e^{\hat{f}_\theta (x,y)}$ in the second term of $I_{\rm DV}$ to control the variance during the inference stage. Although the modification introduces some bias for MI estimation, it is empirically admitting a small variance, which we also find in our experiments.

\subsection{Architecture Design in Experiments}

We follow the same training and evaluation protocal for Correlated Gaussians experiments in prior work~\cite{poole2019variational,song2019understanding}. We adopt the ``concatenate critic'' design~\cite{oord2018representation,poole2019variational,song2019understanding} for our neural network parametrized function. The neural network parametrized functions are $\hat{c}_\theta$ in {\bf CPC}, $\hat{f}_\theta$ in {\bf NWJ}/{\bf JS}/{\bf DV}/{\bf SMILE}/Variational MI Bounds/Density Matching I/Density Matchinig II, $\hat{r}_\theta$ in Density-Ratio Fitting, and $\hat{p}_\theta$ in Probabilistic Classifier. Take $\hat{c}_\theta$ as an example, the concatenate critic design admits $\hat{c}_\theta(x,y) = {{g}_\theta ([x,y])}$  with ${g}_\theta$ being multiple-layer perceptrons. We consider ${g}_\theta$ to be 1-hidden-layer neural network with $512$ neurons for each layer and ReLU function as the activation. The optimization considers batch size $128$ and Adam optimizer~\cite{kingma2014adam} with learning rate $0.001$. For a fair comparison, we fix everything except for the learning and inference objectives. Note that Probabilistic Classifier method applies sigmoid function to the outputs to ensure probabilistic outputs. We set $\eta=1.0$ in Density Matching II. 

\paragraph{Reproducibility} Please refer to our released code.

\subsection{Theoretical Analysis}
We restate the Assumptions in the main text:
\begin{assumption}[Boundedness of the density ratio; restating Assumption 1 in the main text]
There exist universal constants $C_l \leq C_u$ such that $\forall \hat{r}_\theta\in\FF$ and $\forall x,y$, $C_l\leq \log\hat{r}_\theta(x, y)\leq C_u$.
\label{assu:bound_restate}
\end{assumption}
\begin{assumption}[$\log$-smoothness of the density ratio; restating Assumption 2 in the main text]
There exists $\rho > 0$ such that for $\forall x,y$ and $\forall \theta_1, \theta_2\in\Theta$, $|\log\hat{r}_{\theta_1}(x, y) - \log\hat{r}_{\theta_2}(x, y)|\leq \rho\cdot \|\theta_1 - \theta_2\|$.
\label{assu:smooth_restate}
\end{assumption}

In what follows, we first prove the following lemma. The main idea is from~\citet{bartlett1998sample}, while here we focus on the covering number of the parameter space $\Theta$ using $L_2$ norm.
\begin{lemma}[estimation; restating Lemma 1 in the main text] Let $\eps > 0$ and $\covering(\Theta, \eps)$ be the covering number of $\Theta$ with radius $\eps$ under $L_2$ norm. Let $P_{X,Y}$ be any distribution where $S = \{x_i, y_i\}_{i=1}^n$ are sampled from and define $M\defeq C_u - C_l$, then 
\begin{equation}
    \Pr_S\left(\sup_{\hat{r}_\theta\in\FF} \left|\widehat{I}_\theta^{(n)}(X;Y) - \Exp_{P_{X,Y}}[\log \hat{r}_\theta(x, y)]\right| \geq \eps\right) \leq 2\mathcal{N}(\Theta, \eps/ 4\rho)\exp\left(-\frac{n\eps^2}{2M^2}\right).
\end{equation}
\label{lemma:1_restate}
\end{lemma}
\begin{proof}
Define $l_S(\theta)\defeq \widehat{I}_\theta^{(n)}(X;Y) - \Exp_{P_{X,Y}}[\log \hat{r}_\theta(x, y)]$. For $\theta_1, \theta_2\in\Theta$, we first bound the difference $|l_S(\theta_1) - l_S(\theta_2)|$ in terms of the distance between $\theta_1$ and $\theta_2$. To do so, for any joint distribution $P$ over $X\times Y$, we first bound the following difference:
\begin{align*}
    \left|\Exp_P[\log\hat{r}_{\theta_1}(x, y)] - \Exp_P[\log\hat{r}_{\theta_2}(x, y)]\right| &\leq \Exp_P[|\log\hat{r}_{\theta_1}(x, y) - \log\hat{r}_{\theta_2}(x, y)|] \\
    &\leq \Exp_P[\rho\cdot \|\theta_1 - \theta_2\|_2] \\
    &= \rho\cdot \|\theta_1 - \theta_2\|_2,
\end{align*}
where the first inequality is due to the triangle inequality and the second one is from Assumption~\ref{assu:smooth_restate}. Next we bound $|l_S(\theta_1) - l_S(\theta_2)|$ by applying the above inequality twice:
\begin{align*}
    |l_S(\theta_1) - l_S(\theta_2)| &= \left|\left(\widehat{I}_{\theta_1}^{(n)}(X;Y) - \Exp_{P_{X,Y}}[\log \hat{r}_{\theta_1}(x, y)]\right) - \left(\widehat{I}_{\theta_2}^{(n)}(X;Y) - \Exp_{P_{X,Y}}[\log \hat{r}_{\theta_2}(x, y)]\right)\right| \\
    &\leq \left|\widehat{I}_{\theta_1}^{(n)}(X;Y) - \widehat{I}_{\theta_2}^{(n)}(X;Y)\right| + \left|\Exp_{P_{X,Y}}[\log \hat{r}_{\theta_1}(x, y)] - \Exp_{P_{X,Y}}[\log \hat{r}_{\theta_2}(x, y)]\right| \\
    &\leq \rho\cdot \|\theta_1 - \theta_2\| + \rho\cdot \|\theta_1 - \theta_2\|_2 \\
    &= 2\rho\cdot \|\theta_1 - \theta_2\|.
\end{align*}
Now we consider the covering of $\Theta$. Since $\Theta$ is compact, it admits a finite covering. To simplify the notation, let $T\defeq \mathcal{N}(\Theta, \eps/4\rho)$ and let $\cup_{k = 1}^T \Theta_k$ be a finite cover of $\Theta$. Furthermore, assume $\theta_i\in\Theta_i$ be the center of the $L_2$ ball $\Theta_i$ with radius $\eps / 4\rho$. As a result, the following bound holds:
\begin{align*}
    \Pr_{S}(\sup_{\hat{r}_\theta\in\FF}|l_S(\theta)| \geq\eps) &= \Pr_{S}(\sup_{\theta\in\Theta}|l_S(\theta)| \geq\eps) \\
    &\leq \Pr_{S}(\cup_{k\in[T]}\sup_{\theta\in\Theta_k}|l_S(\theta)|\geq \eps) \\
    &\leq \sum_{k\in[T]}\Pr_{S}(\sup_{\theta\in\Theta_k}|l_S(\theta)|\geq\eps).
\end{align*}
The last inequality above is due to the union bound. Next, $\forall k\in[T]$, realize that the following inequality holds:
\begin{equation*}
    \Pr_{S}(\sup_{\theta\in\Theta_k}|l_S(\theta)|\geq\eps) \leq \Pr_{S}(|l_S(\theta_k)|\geq\eps / 2). 
\end{equation*}
To see this, note that the $L_2$ ball of $\Theta_k$ has radius $\eps / 4\rho$, hence $\sup_{\theta\in\Theta_k}|l_S(\theta) - l_S(\theta_k)|\leq 2\rho \cdot \eps / 4\rho = \eps / 2$, which yields:
\begin{align*}
    \Pr_{S}(\sup_{\theta\in\Theta_k}|l_S(\theta)|\geq\eps) &\leq \Pr_{S}(\sup_{\theta\in\Theta_k}|l_S(\theta) - l_S(\theta_k)| + |l_S(\theta_k)|\geq\eps) \\
    &\leq \Pr_{S}(|l_S(\theta_k)|\geq\eps / 2). 
\end{align*}
To proceed, it suffices if we could provide an upper bound for $\Pr_{S}(|l_S(\theta_k)|\geq\eps / 2)$. Now since $\log\hat{r}_{\theta_k}(x, y)$ is bounded for any pair of input $x, y$ by Assumption~\ref{assu:bound_restate}, it follows from the Hoeffding's inequality that
\begin{align*}
    \Pr_{S}(|l_S(\theta_k)|\geq\eps / 2) &= \Pr_{S}\left(\left|\widehat{I}_{\theta_k}^{(n)}(X;Y) - \Exp_{P_{X,Y}}[\log \hat{r}_{\theta_k}(x, y)]\right|\geq\eps / 2\right) \\
    &\leq 2\exp\left(-\frac{n\eps^2}{2M^2}\right).
\end{align*}
Now, combine all the pieces together, we have:
\begin{align*}
    \Pr_S(\sup_{\hat{r}_\theta\in\FF} \left|\widehat{I}_\theta^{(n)}(X;Y) - \Exp_{P_{X,Y}}[\log \hat{r}_\theta(x, y)]\right|\geq\eps) &= \Pr_S(\sup_{\theta\in\Theta} \left|l_S(\theta)\right|\geq\eps) \\
    &\leq \sum_{k\in[T]}\Pr_{S}(\sup_{\theta\in\Theta_k}|l_S(\theta)|\geq\eps) \\
    &\leq \mathcal{N}(\Theta, \eps/ 4\rho)\Pr_{S}(\sup_{\theta\in\Theta_k}|l_S(\theta)|\geq\eps) \\
    &\leq \mathcal{N}(\Theta, \eps/ 4\rho)\Pr_{S}(|l_S(\theta_k)|\geq\eps / 2) \\
    &\leq 2\mathcal{N}(\Theta, \eps/ 4\rho)\exp\left(-\frac{n\eps^2}{2M^2}\right).\qedhere
\end{align*}
\end{proof}

We restate the Lemma 2 in the main text:
\begin{lemma}[\citet{hornik1989multilayer}, approximation; restating Lemma 2 in the main text]
Let $\eps > 0$. There exists $d \in\Nat$ and a family of neural networks $\FF\defeq \{\hat{r}_\theta: \theta\in\Theta\subseteq\RR^d\}$ where $\Theta$ is compact, such that $\inf_{ \hat{r}_\theta\in\FF}\left|\Exp_{P_{X, Y}}[\log \hat{r}_\theta(x, y)] - I(X; Y)\right|\leq \eps$.
\label{lemma:2_restate}
\end{lemma}

Now, we are ready the present our theorem:
\begin{theorem}
Let $0 < \delta < 1$. There exists $d \in\Nat$ and a family of neural networks $\FF\defeq \{\hat{r}_\theta: \theta\in\Theta\subseteq\RR^d\}$ where $\Theta$ is compact, so that $\exists \theta^*\in\Theta$, with probability at least $1 - \delta$ over the draw of $S = \{x_i, y_i\}_{i=1}^n\sim P_{X,Y}^{\otimes n}$, 
\begin{equation}
    \left|\widehat{I}_{\theta^*}^{(n)}(X;Y) - I(X; Y)\right| \leq O\left(\sqrt{\frac{d + \log(1/\delta)}{n}}\right).
\end{equation}
\end{theorem}
\begin{proof}
This theorem simply follows a combination of Lemma~\ref{lemma:1_restate} and Lemma~\ref{lemma:2_restate}. First, by Lemma~\ref{lemma:2_restate}, for $\eps > 0$, there exists $d \in\Nat$ and a family of neural networks $\FF\defeq \{\hat{r}_\theta: \theta\in\Theta\subseteq\RR^d\}$ where $\Theta$ is compact, such that there $\exists\theta^*\in\Theta$, 
\begin{equation*}
\left|\Exp_{P_{X, Y}}[\log \hat{r}_{\theta^*}(x, y)] - I(X; Y)\right|\leq \frac{\eps}{2}.    
\end{equation*}

Next, we perform analysis on the estimation error $\left|\widehat{I}_{\theta^*}^{(n)}(X;Y) - \Exp_{P_{X, Y}}[\log \hat{r}_{\theta^*}(x, y)]\right| \leq \frac{\eps}{2}$. Applying Lemma~\ref{lemma:1_restate} with the fact~\citep{anthony2009neural} that for $\Theta\subseteq\RR^d$, $\log\mathcal{N}(\Theta, \eps / 4\rho) = O(d\log(\rho / \eps))$, we can solve for $\eps$ in terms of the given $\delta$. It suffices for us to find $\eps \rightarrow \frac{\eps}{2}$ such that:
\begin{equation*}
    2\mathcal{N}(\Theta, \eps/ 8\rho)\exp\left(-\frac{n\eps^2}{8M^2}\right) \leq \delta,
\end{equation*}
which is equivalent to finding $\eps$ such that the following inequality holds:
\begin{equation*}
    c\cdot d\log\frac{\eps}{8\rho} + \frac{n\eps^2}{8M^2} \geq \log\frac{2}{\delta},
\end{equation*}
where $c$ is a universal constant that is independent of $d$. Now, using the inequality $\log(x)\leq x-1$, it suffices for us to find $\eps$ such that
\begin{equation*}
    c\cdot d\left(\frac{\eps}{8\rho} - 1\right) + \frac{n\eps^2}{8M^2} \geq c\cdot d\log\frac{\eps}{8\rho} + \frac{n\eps^2}{8M^2} \geq \log\frac{2}{\delta},
\end{equation*}
which is in turn equivalent to solving:
\begin{equation*}
    \eps^2 + c'\eps \geq \left(\log\frac{2}{\delta} + cd\right)\cdot\frac{8M^2}{n}, 
\end{equation*}
where $c' = c'(c, d, \rho, n, M)$. Nevertheless, in order for the above inequality to hold, it suffices if we choose 
\begin{equation*}
    \eps = O\left(\sqrt{\frac{d + \log(1/\delta)}{n}}\right).
\end{equation*}
The final step is to combine the above two inequalities together:
\begin{align*}
\left|\widehat{I}_{\theta^*}^{(n)}(X;Y) - I(X; Y)\right| &\leq \left|\widehat{I}_{\theta^*}^{(n)}(X;Y) - \Exp_{P_{X, Y}}[\log \hat{r}_{\theta^*}(x, y)]\right| + \left|\Exp_{P_{X, Y}}[\log \hat{r}_{\theta^*}(x, y)] - I(X; Y)\right| \\
&\leq \frac{\eps}{2} + \frac{\eps}{2} = O\left(\sqrt{\frac{d + \log(1/\delta)}{n}}\right).\qedhere
\end{align*}
\end{proof}

\section{More on Self-supervised Representation Learning}

In the main text, we have shown how we adapt the proposed point-wise dependency estimation approaches (Probabilistic Classifier and Density-Ratio Fitting) to contrastive learning objectives (Probabilistic Classifier Coding and Density-Ratio Fitting Coding) for self-supervised representation learning. Following the adaptation, it is straightforward to define new contrastive learning objectives that are inspired by other presented approaches such as Variational MI Bounds, Density Matching I ,and Density Matching II. Nevertheless, instead of presenting new objectives, we would like to discuss 1) the connection between Probabilistic Classifier and Variational MI Bounds; 2) the connection between Density Matchinig I/II  and $I_{\rm NWJ}$ (see Proposition~\ref{prop:nwj}); and 3) the potential limitations of the new objectives. Next, we will discuss the baseline method Contrastive Predictive Coding (CPC). Last, we present the experimental details.

\subsection{Connection between Probabilistic Classifier and Variational MI Bounds}
\label{subsec:pc_vmib}

Proposition~\ref{prop:pc} states that the Probabilistic Classifier approach admits a classification task to differentiate the pairs sampled from a joint distribution or the product of marginal distribution. This classification task minimizes the binary cross entropy loss, which is highly optimized and stabilized in popular optimization packages such as PyTorch~\cite{paszke2019pytorch} and TensorFlow~\cite{abadi2016tensorflow} (e.g., log-sum-exp trick for numerical stability). Note that, if we let $\hat{p}_\theta = {\rm sigmoid} \,\big(l_\theta \big)$ with $l_\theta$ being the logits model, then reformulating Probabilistic Classifier to optimizing $l_\theta$ leads to the same objective as $I_{\rm JS}$ (see Proposition~\ref{prop:js}), which is the learning objective of {\em Variational MI Bounds} method. Although being the same objective as the Probabilistic Classifier approach, $I_{\rm JS}$ may encounter a relatively higher training instability (unless a particular take-care on its numerical instability). As pointed out by Tschannen {\em et al.}~\cite{tschannen2019mutual}, contrastive learning approaches with higher variance may result in a lower down-stream task performance, which accords with our empirical observation.

\subsection{Connection between Density Matching I/II and $I_{\rm NWJ}$}

Density Matching I/II approaches are derived from the KL loss between the true joint density and estimated joint density ($L_{\rm KL_{DM}}$ in Proposition~\ref{prop:KL_DM}). Specifically, Density Matching I is a Lagrange relaxation of $L_{\rm KL_{DM}}$. If we change $\hat{f}_\theta + 1= \hat{f}'_\theta$ in Density Matching I approach, then reformulating our objective to optimizing $\hat{f}'_\theta$ leads to the same objective as $I_{\rm NWJ}$ (see Proposition~\ref{prop:nwj}). Song {\em et al.}~\cite{song2019understanding} acknowledged the variance of an empirical $I_{\rm NWJ}$ estimation is $\Omega (e^{I(X;Y)})$, and hence the variance is large unless $I(X;Y)$ is small. Having the same conclusion in Sec~\ref{subsec:pc_vmib}, our empirical observation finds Density Matching I/II lead to worsened representation as comparing to other contrastive learning objectives.

\subsection{Contrastive Predictive Coding (CPC) for Contrastive Representation Learning}

Contrastive Predictive Coding (CPC)~\cite{oord2018representation} adapts $I_{\rm CPC}$ (see Proposition~\ref{prop:cpc}) to a contrastive representation learning objective:
\begin{equation*}
\begin{split}
 \underset{F,G}{\rm sup}\,\,\underset{\theta \in \Theta}{\rm sup}\,\,  \mathbb{E}_{(v_1^1,v_2^1)\sim P_{\mathcal{V}_1,\mathcal{V}_2}, \cdots (v_1^n,v_2^n)\sim P_{\mathcal{V}_1,\mathcal{V}_2}}[\frac{1}{n}\sum_{i=1}^{n}{\rm log}\,\frac{e^{\hat{c}_\theta(F(v_1^i), G(v_2^i))}}{\frac{1}{n}\sum_{j=1}^{n}e^{\hat{c}_\theta(F(v_1^i), G(v_2^j))}}],
\end{split}
\end{equation*} 
where $\{v_1^i, v_2^i\}_{i=1}^n$ are independently and identically sampled from $P_{\mathcal{V}_1,\mathcal{V}_2}$. $\hat{c}_\theta (\cdot)$ is a function that takes the representations learned from the data pairs and returns a scalar.

\subsection{Experiments Details}

\paragraph{Datasets} We adopt  MNIST~\cite{lecun1998gradient} and CIFAR10~\cite{krizhevsky2009learning} as the datasets in our experiments. MNIST contains $60,000$ training and $10,000$ test examples. Each example is a grey-scale digit image ($0 \sim 9$) with size $28\times 28$. CIFAR10 contains $50,000$ training and $10,000$ test examples. Each example is a $32 \times 32$ colour image from $10$ mutual exclusive classes: \{airplane, automobile, bird, cat, deer, dog, frog, horse, ship, truck\}. 

\paragraph{Pre-training and Fine-tuning} Our self-supervised learning experiments contain two stages: {\em pre-training} and {\em fine-tuning}. In pre-training stage, we learn representation from the training samples using contrastive learning objectives (e.g., Probabilistic Classifier Coding (PCC), Density-Ratio Fitting Coding (D-RFC), and Contrastive Predictive Coding (CPC)~\cite{oord2018representation}). View 1 ($\mathcal{V}_1$) and 2 ($\mathcal{V}_2$) are generated by augmenting the input with different transformations. For example, given an input, $v_1$ can be the 15-degree-rotated one and $v_2$ can be the horizontally flipped one. For {\bf shallow} experiment, we consider the same data augmentations adopted in Tschannen~\cite{tschannen2019mutual}; for {\bf deep} experiment, we consider the same data augmentations adopted in Bachman~\cite{bachman2019learning}. In fine-tuning stage, the network in the pre-training stage is fixed; we train only the classifier for minimizing classification loss from the  representations. We follow linear evaluation protocol~\cite{oord2018representation,henaff2019data,tian2019contrastive,hjelm2018learning,bachman2019learning,tschannen2019mutual,kolesnikov2019revisiting} such that the classifier is a linear layer. After the pre-training and fine-tuning stages, we evaluate the performance of the model on the test samples.

\paragraph{Architectures} To clearly understand how contrastive learning objectives affect the down-stream performance, we fix the network, learnnig rate, optimizer, and batch size across different objectives. To be more precise, we stick to the official implementations by Tschannen {\em et al.}~\cite{tschannen2019mutual} (for {\bf shallow} experiment) and Bachman {\em et al.}~\cite{bachman2019learning} (for {\bf deep} experiment). The only change is the contrastive learning objective, which is the loss in the pre-training stage for self-supervised learning experiments. 

\paragraph{Reproducibility} One can refer to \url{https://github.com/google-research/google-research/tree/master/mutual_information_representation_learning} and \url{https://github.com/Philip-Bachman/amdim-public} for the authors' official implementations, or checking the details in our released code. 

\paragraph{Consistent Trend on SimCLR~\cite{chen2020simple}} We also evaluate CPC, PCC, and D-RFC in SimCLR~\cite{chen2020simple}, which is a SOTA model and method on self-supervised representation learning. Note that the default contrastive learning objective considered in SimCLR~\cite{chen2020simple} is CPC, which obtains $91.04\%$ test accuracy on CIFAR-10 (average for $5$ runs). Details can be found in \url{https://github.com/google-research/simclr}. Similar to our {\bf shallow} and {\bf deep} experiments, we only change the contrastive learning objectives in SimCLR, and observing $91.51\%$ and $88.69\%$ average test accuracy for D-RFC and PCC, respectively. The trend is consistent with our {\bf deep} experiment, where D-RFC works slightly better than CPC and PCC works slightly worse than CPC.

\section{More on Cross-Modal Learning}

\begin{figure}[t!]
\centering
\includegraphics[width=0.8\linewidth]{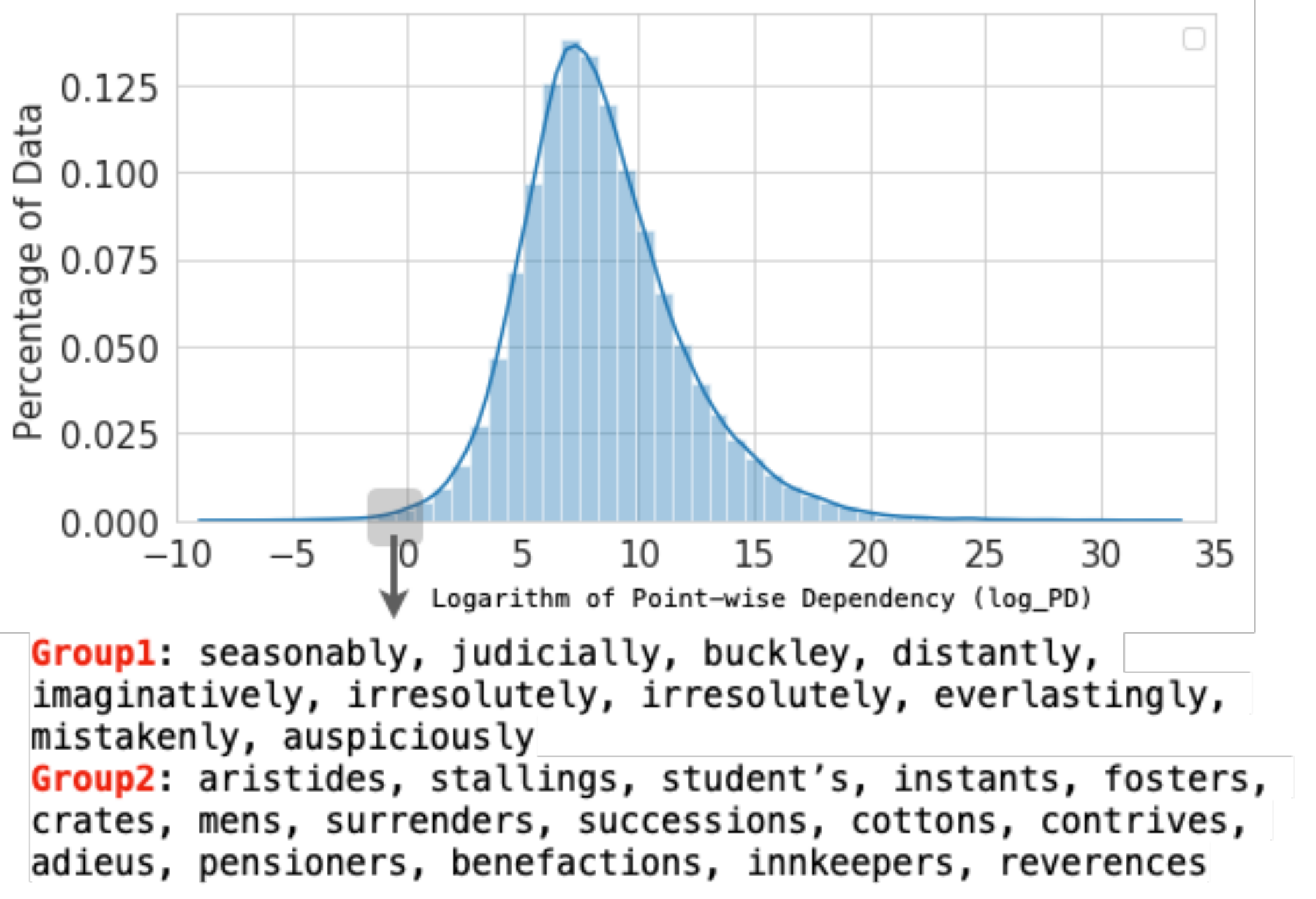}
  \caption{\small 
  {\bf Dataset Debugging} task with  unsupervised word features across acoustic and textual modalities. {\em Probabilistic Classifier} approach is used to estimate PD between the audio and textual feature of a given word. The estimator is trained on the training split.
  We plot the logarithm of PD (i.e., PMI) distribution for the training words. We select the words with negative PMI values and categorize them into two groups: one contains the words end in ``ly'' and another containts the words end in ``s''. 
  }
\label{fig:CM}
\vspace{-5mm}
\end{figure}

\paragraph{Another Case Study: Cross-modal Adversarial Samples Debugging} One important topic in interpretable machine learning~\cite{molnar2019interpretable} is dataset debugging, which detects adversarial samples in a given dataset. For instance, in this dataset, an adversarial word feature would have low statistical dependency between its audio and textual representations. In Fig.~\ref{fig:CM}, we report the PMI distribution and highlight the training words with PMI $<0$ (i.e., the adversarial samples). We note that a negative PMI means the audio and textual features are either statistically independent or even co-occur less frequently than the independent assumption. 

First, we find the distribution of PMI resembles a Gaussian distribution. The mean of the PMI values is MI, and our empirical estimation for it is 8.37. Our goal is to identify the training samples with PMI that deviates far from MI, and especially for the samples with negative PMI. There are 147 words have negative PMI values, approximately 0.45\% of the training words. Next, we select some of these words and categorize them into two groups. The first group contains the words end in ``ly'' and another group contains the words end in ``s''. That is to say, the words end in ``ly'' and ``s'' are adversarial training sample in our analysis. To sum up, we demonstrate how our PD estimation approach can be used to detect adversarial training examples in a cross-modal dataset.

\paragraph{Dataset} We construct a dataset that contains features from Word2Vec~\cite{mikolov2013efficient} and Speech2Vec~\cite{chung2018speech2vec}. Word2Vec is an unsupervised word embedding learning technique that takes a large text corpus of text as input and produces a fixed-length vector space. Specifically, each word in the corpus is assigned a real-valued and fixed-dimensional feature embedding. Similar to Word2Vec, Speech2Vec takes a large corpus of human speech as input and produces a fixed-length vector space. Specifically, it transforms a variable-length speech segment (a word in the speech corpus) as a real-valued and fixed-dimensional feature embedding. There are $37,622$ words shared across Word2Vec and Speech2Vec, where we consider $32,622$ words of them (randomly selected) to be the training split and $5,000$ of them to be the test split. That is to say, each word contains a textual feature (from Word2Vec) and an audio feature (from Speech2Vec), with both feature being $100-$dimensional. The dataset can be downloaded from \url{https://github.com/iamyuanchung/speech2vec-pretrained-vectors} and we include the training/test split in our released code.

\paragraph{Training and Architectures} We adopt the ``separate critic'' design~\cite{oord2018representation,poole2019variational,song2019understanding} for our neural network parametrized function. Suppose $\hat{l}_\theta$ is the logits model in Probabilistic Classifier approach, and the separate critic design admits $\hat{l}_\theta(x,y) = {{g_x}_\theta (x)}^\top {g_y}_\theta (y)$ with ${g_x}_\theta$ and ${g_y}_\theta$ being different multiple layer perceptrons. We consider ${g_x}_\theta$ and ${g_y}_\theta$ to be 1-hidden-layer neural network with $512$ neurons for intermediate layers, $128$ neurons for the output layer, and ReLU function as the activation. The optimization considers batch size $512$ and Adam optimizer~\cite{kingma2014adam} with learning rate $0.001$. A sigmoid function is applied to $\hat{l}_\theta$ ($\hat{p}_\theta = {\rm sigmoid} (\hat{l}_\theta)$) to ensure $\hat{p}_\theta$ is a probabilistic output. We consider $100$ training epochs.

\paragraph{Reproducibility} Please refer to our released code, where we also include the dataset and its training/ test split.
\section{Practical Deployment for Expectation(s)}

In practice, the expectations in Propositions~\ref{prop:nwj},~\ref{prop:dv},~\ref{prop:js},~\ref{prop:KL_DM},~\ref{prop:pc},~\ref{prop:LS_D-RF}, and~\ref{prop:cpc} are estimated using empirical samples from $P_{X,Y}$ and $P_{X}P_Y$. With mild assumptions on the compactness of $\Theta$ and the boundness of our measurement, the estimation error would be small by uniform law of large numbers~\cite{van2000asymptotic}.

\end{document}